\documentclass[runningheads]{llncs}
\usepackage{natbib}
\usepackage{graphicx}
\usepackage[margin=1.0in]{geometry}
\usepackage{booktabs}
\usepackage{bm}

\usepackage{amsmath,amssymb,paralist}
\usepackage{dsfont}
\usepackage{mathabx}
\usepackage{subcaption}
\usepackage{microtype}
\usepackage{cleveref}
\usepackage{thmtools, thm-restate}

\newcommand\numberthis{\addtocounter{equation}{1}\tag{\theequation}}

\newcommand{\apm}{Approximate Pattern Mining}
\newcommand{\slt}{Statistical Learning Theory}

\newcommand{\pars}[1]{\left( #1 \right)}
\newcommand{\sqpars}[1]{\left[ #1 \right]}
\newcommand{\brpars}[1]{\left\lbrace #1 \right\rbrace}
\newcommand{\qt}[1]{\lq\lq#1\rq\rq}

\newcommand{\sample}{\mathcal{S}}

\newcommand{\BLi}[1]{\Omega ( #1 )}
\newcommand{\BO}[1]{\mathcal{O}\left(#1\right)}
\newcommand{\BOi}[1]{\mathcal{O}( #1 )}

\newcommand{\probdist}{\mu}

\newcommand{\E}{\mathbb{E}}
\newcommand{\F}{\mathcal{F}}
\newcommand{\G}{\mathcal{G}}

\newcommand{\mcera}{\text{MCERA}}
\newcommand{\nmcera}{n\text{-\mcera}}

\newcommand{\abs}[1]{\lvert#1 \rvert}

\newcommand{\frange}{c}

\newcommand{\R}{\mathbb{R}}
\newcommand{\rade}{\mathsf{R}}
\newcommand{\erade}{\hat{\rade}}

\newcommand{\rc}{\rade(\F, m)}

\newcommand{\era}{\erade\left(\F, \sample\right)}

\newcommand{\sd}{\mathsf{D}}

\newcommand{\supdevpos}{\sd^{+}(\F, \sample)}
\newcommand{\supdevneg}{\sd^{-}(\F, \sample)}

\newcommand{\vsigma}{{\bm{\sigma}}}

\newcommand{\X}{\mathcal{X}}
\newcommand{\mera}{\erade^{n}_{m}(\F, \sample, \vsigma)}
\newcommand{\onemera}{\erade^{1}_{m}(\F, \sample, \vsigma)}

\newcommand{\merad}{\erade^{n}_{m}(\F^\prime, \sample, \vsigma)}

\newcommand{\logdelta}{\ln \left( \frac{1}{\delta} \right)}

\newcommand{\maxexpf}{\eta_\F}
\newcommand{\emaxexpf}{\hat{\eta}_\F\pars{\sample}}
\newcommand{\minexpf}{\gamma_\F}
\newcommand{\eminexpf}{\hat{\gamma}_\F\pars{\sample}}

\newcommand{\emaxexpfp}{\hat{\eta}_{\F^\prime}\pars{\sample}}

\newcommand{\maxabsf}{\hat{\nu}_{\F}\pars{\sample}}
\newcommand{\maxabsfd}{\hat{\nu}_{\F^{\prime}}\pars{\sample}}
\newcommand{\erad}{\erade\left(\F^{\prime}, \sample\right)}
\newcommand{\wvar}{\sigma^{2}_{\F}}
\newcommand{\ewvar}{\hat{\sigma}^{2}_{\F}\pars{\sample}}

\newcommand{\emaxf}{\hat{z}\pars{\sample}}

\newcommand{\maxsquaref}{\ewvar}
\newcommand{\maxsquarefd}{\hat{\sigma}^{2}_{\F^{\prime}}\pars{\sample}}

\newif\iflongversion%
\longversionfalse%

\begin{document}
\title{Sharper convergence bounds of Monte Carlo \\Rademacher Averages
through Self-Bounding functions
%\thanks{Supported by x.}
}
\titlerunning{Sharper convergence bounds of Monte Carlo Rademacher Averages
through Self-Bounding functions}
\author{Leonardo Pellegrina
}
\authorrunning{Leonardo Pellegrina}
\institute{
\email{pellegri@dei.unipd.it}, \\ 
Department of Information Engineering, \\
University of Padova.\\~\\
January 16, 2021}

\maketitle              % typeset the header of the contribution
\begin{abstract}
We derive sharper probabilistic concentration bounds for the Monte Carlo Empirical Rademacher Averages (MCERA), which are proved through recent results on the concentration of self-bounding functions. 
Our novel bounds are characterized by convergence rates that depend on data-dependent characteristic quantities of the set of functions under consideration, such as the empirical wimpy variance, an essential improvement w.r.t. standard bounds based on the methods of bounded differences. For this reason, our new results are applicable to yield sharper bounds to (Local) Rademacher Averages. 
We also derive improved novel variance-dependent bounds for the special case where only one vector of Rademacher random variables is used to compute the MCERA, through the application of Bousquet's inequality and novel data-dependent bounds to the wimpy variance. Then, we leverage the framework of self-bounding functions to derive novel probabilistic bounds to the supremum deviations, that may be of independent interest.

\end{abstract}
%
%
%

%\section{Previous Work}

% !TEX root = mcrade-bounds.tex

\section{Introduction}

Uniform convergence is a central problem in \slt~\citep{Vapnik98}.
Obtaining tight and uniformly valid probabilistic bounds on the accuracy of empirical averages of sets of functions is a fundamental problem, with widespread and impactful applications in Machine Learning and Data Science~\citep{anthony2009neural,ShalevSBD14,
mitzenmacher2017probability,
mohri2018foundations}.
Probabilistic bounds on the largest error of the empirical averages are typically obtained by adding to the empirically estimated error a term that depends on the complexity of the functions. Both distribution-free concepts of complexity, such as the VC dimension~\citep{Vapnik:1971aa}, and distribution and data-dependent complexities have been proposed as breakthroughs with great success for this problem 
(\citep{
KoltchinskiiP00,
mendelson2002improving,
bartlett2002model,
BartlettM02,
blanchard2003rate,
bartlett2005local,
Koltchinskii06,
blanchard2008statistical,
gnecco2008approximation,
kloft2011local,
anguita2012sample,
Cortes2013,
oneto2015local,
lei2015multi,
oneto2017learning,
kuznetsov2017generalization,
yousefi2018local,
yin2019rademacher,
lei2019data,
musayeva2019rademacher} and many others).
%In this Chapter
In this work
 we provide new convergence bounds for one of the most interesting notions of data-dependent measure of complexity of sets of functions, the Rademacher Complexity. In particular, we show that it can be estimated in a Monte Carlo way obtaining ``faster convergence rates'' that depend on characteristic quantities of the set of functions and the sample.
%, extensively studied in the context of uniform convergence by~\cite{KoltchinskiiP00}, \cite{bartlett2002model}, \cite{BartlettM02} and others.

A potential drawback of Rademacher Averages is that the ``global'' error that can be obtained may be characterised by the so called ``slow'' convergence rate of 
$\BOi{m^{-1/2}}$, where $m$ is the number of analysed samples; while such rate is essentially the best possible when some elements of a set of function $\F$ achieve maximum variance~\citep{boucheron2005theory}, it may be substantially improved for the other functions, that are often more interesting to the analysis.
Therefore, a rich collection of 
contributions~\citep{KoltchinskiiP00,Massart00,
bousquet2002some,
mendelson2002improving,
bartlett2005local,
Koltchinskii06,
koltchinskii2011oracle,
mendelson2014learning} 
have then focused on providing \emph{local} estimates of the complexity, restricting the estimation to a proper subset of $\F$ that contains only functions with lower variance. 
%This operation is also known as ``peeling'', and it is often combined with scaling the elements of $\F$ with weights that are functions of their variance. 
In such settings, one would hope to achieve sharper error bounds, with rates between $\BOi{m^{-1/2}}$ and $\BOi{m^{-1}}$.

The slow convergence rate can be attributed to both the ``global'' computation of Rademacher Averages and from the application of probabilistic concentration inequalities based on the method of bounded differences, that is essentially tight only when there are elements of the set of functions under consideration that achieve maximum variance~\citep{boucheron2013concentration}. Therefore, the study of novel concentration inequalities for the supremum of empirical processes that take advantage of smaller bounds to the variance has been a central focus of research, such as the fundamental contributions of \cite{talagrand1994sharper,talagrand1995concentration} and many others \citep{boucheron2000sharp,
bousquet2002bennett,
boucheron2005theory,
boucheron2013concentration}.

The standard approach to bound the Rademacher Complexity is through the application of Massart's Lemma~\citep{Massart00}. 
%, that yields a \emph{deterministic} upper bound to it. 
An alternative, often much sharper, approach is to \emph{directly estimate} the Rademacher Averages with the \emph{$n$-Monte Carlo Empirical Rademacher Average} ($\nmcera$) (defined formally in the next Section); 
this quantity is computed by sampling a finite number of vectors of Rademacher random variables, instead of evaluating its expectation~\citep{BartlettM02}, 
% matrix of $n \times m$ Rademacher random variables and then computing the corresponding $n$ supremums over the elements of $\F$ on the $m$ samples~\citep{BartlettM02}, 
 and then obtaining a probabilistic upper bound to the Rademacher Complexity with concentration of measure inequalities. 

In a recent work, \citet{de2019rademacher} used the framework of uniform convergence and Rademacher Complexity to obtain error bounds to empirical averages in an adaptive setting: in their scenario, batch of functions are considered at successive steps, while allowing the choice of the functions to process at every iteration to be based on past information. 
To quantify the risk of \qt{overfitting}, they leverage the $\nmcera$, computing it efficiently as functions are processed. 
Their analysis relates the $\nmcera$ to its expectation through Bernstein's inequality and on the Central Limit Theorem for martingales. 

In other situations, in particular when the size of $\F$ is large, it may be more expensive to compute the $\nmcera$, 
%and it is probably the reason why this approach has not gained 
limiting a more 
widespread practical consideration. 
In the context of Data Mining and \apm, 
\citet{pellegrina2020mcrapper} address this computational challenge, deriving 
%we have shown in \Cref{chap:mcrapper} that \mcrapper\ is 
a general and practical scheme to compute the $\nmcera$ by exploiting the combinatorial structure of $\F$ in a branch-and-bound strategy. In all these applications, it is critical to apply sharp concentration results to have tight error rates. 

The works we described \citep{de2019rademacher,pellegrina2020mcrapper} achieve error bounds that relate the $\nmcera$ to its expectation, the \emph{Empirical Rademacher Average} (ERA), using concentration inequalities based on the bounded difference property (or, equivalently, assuming maximum variance); for this reason, such error bounds are characterised by the slow convergence rate of $\BOi{(nm)^{-1/2}}$, analogous to the worst-case rate of uniform convergence we discussed before. While, in theory, one could use an arbitrary large number $n$ of vectors of Rademacher random variables, and in particular $n=m$ to achieve $\BOi{m^{-1}}$ error rates for estimating the ERA, this would imply the computation of a large number of supremums over $\F$, something impractical in almost all situations.

The question of whether the $\nmcera$ can be tightly estimated without using an impractically large number of Monte Carlo trials is an unexplored question. In fact, sharp variance-dependent concentration inequalities that relate the $\nmcera$ to its expectation are not available.

\textbf{Our contributions.}
The main goal of this work is to provide a positive answer to this question: in Section~\ref{sec:newboundserasb} we derive novel concentration bounds for the $\nmcera$ whose convergence rates depend on characteristic quantities computable from the data, such as the \emph{empirical wimpy-variance} of the set of functions, resulting in a significantly improved trade-off between the guaranteed convergence of the estimate and the number $n$ of required vectors of Rademacher random variables. 
To do so, we first establish, in Section~\ref{sec:nmceraselfboundingproperties}, \emph{self}-\emph{bounding} properties of the $\mcera$.
Then, we leverage such properties to derive, in Section~\ref{sec:selfberabounds}, novel concentration inequalities for the $\mcera$ w.r.t. its expectation, the ERA; such results follow from the sharp exponential concentration inequalities that self-bounding functions satisfy~\citep{boucheron2000sharp,boucheron2009concentration}. 
Furthermore, in Section~\ref{sec:newboundsnone} we study the special case of $n=1$, and prove a novel concentration inequality that directly relates the $\mcera$ to the Rademacher Complexity, though the application of Bousquet's inequality~\citep{bousquet2002bennett}, a central result in \slt. As the rate of convergence of such bound depends on the unknown \emph{wimpy variance} of the set of functions $\F$, we show that it can be tightly estimated from the available data using its empirical counterpart, the empirical wimpy variance. The guaranteed accuracy of such empirical estimator is proved with the powerful framework of self-bounding functions.

The new bounds we derive in this work are relevant to all methods based on the $\nmcera$ we introduced before and, given their generality, possibly others. In particular, we believe it would be interesting to fit our results in the framework of Localised Rademacher Averages, and that there are interesting new algorithmic applications of the $\nmcera$ that may benefit from our results, in particular in problems already tackled with methods based on Rademacher Averages; examples are the analysis of large networks~\citep{RiondatoU18,de2020estimating}, rigorous Pattern Mining ~\citep{RiondatoU15,santoro2020mining} Statistical Hypothesis Testing~\citep{PellegrinaRV19a,li2019multiple}, and, potentially, many others.

Another interesting question we explore is whether the maximum difference between empirical averages and their expectation, quantities often denoted by \emph{Supremum Deviations} (SDs), satisfy some form of self-bounding properties. Indeed, after introducing, in Section~\ref{sec:stdvarsdbounds}, the state-of-the-art variance-dependent bounds to the SDs, in Section~\ref{sec:newboundssds} we show that the SDs are also self-bounding, for appropriate constants that depend on the maximum and minimum expected values of the functions in $\F$; consequently, we derive novel concentration inequalities for the SDs, that may be of independent interest. 

We conclude comparing our novel bounds and empirical estimators w.r.t. the state-of-the-art with some simulations, described in Section \ref{sec:experiments}. 

%Lastly, in Section~\ref{sec:newboundsnone} we study the interesting case of $n=1$, and show that it is possible to directly and tightly bound the Rademacher Complexity from its estimate $\onemcera$ with a single application of a novel variance-aware concentration inequality.

%A preliminar version of this works appears at~\citep{pellegrina2020sharper}.

\section{Preliminaries}\label{sec:prelims}
We denote $\F$ to be a class of real valued functions from a domain $\X$ to the
bounded interval $[a , b] \subset \R$, and let $z \doteq \max\{|a|, |b|\}$ and $\frange \doteq b-a$, with $b > 0 \geq a$, and $c , z > 0$. To simply address non-negativity issues, we assume w.l.o.g. that $\F$ contains a constant function $f_{0}$ such that $f_{0}(x) = 0$, for all $x \in \X$.

Let a sample $\sample$ be a bag $\left\lbrace s_{1} , \dots , s_{m} \right\rbrace$ of size $m$, such that $s \in \X, \forall s \in \sample$.
We assume that each element of $\sample$ is drawn i.i.d. from $\X$ according to an unknown probability distribution $\probdist$.
Our goal is to derive tight bounds on the difference between the average value of $f$, computed on the sample $\sample$, and its expectation $\E[f]$, taken w.r.t. $\sample$, that are valid for all functions $f\in \F$. More formally, we define the positive Supremum Deviation (SD) $\supdevpos$ and the negative supremum deviation $\supdevneg$ as 
\begin{align*}
  \supdevpos \doteq \sup_{f \in \F} \left\lbrace \frac{1}{m} \sum_{i=1}^{m}f(s_{i}) - \E[f] \right\rbrace
  \enspace, \enspace
  \supdevneg \doteq \sup_{f \in \F} \left\lbrace \E[f] - \frac{1}{m} \sum_{i=1}^{m}f(s_{i})  \right\rbrace \enspace .
\end{align*}
As $\probdist$ is unknown, it is not possible to directly compute such supremum deviations. However, fundamental results from \slt\ allow to obtain probabilistic upper bounds to them, exploiting information obtainable from the data $\sample$. We introduce the concepts of \emph{Rademacher Averages}, that will be instrumental to achieve this goal.

First, let $\vsigma$ be a $n \times m$ matrix such that each component $\vsigma_{i,j}$ of index $(i,j)$ is either $1$ or $-1$. The \emph{$n$-Monte Carlo Empirical Rademacher Average} ($\nmcera$) $\mera$ is defined as
\begin{equation*}%\label{eq:mcera}
  \mera \doteq \frac{1}{n}\sum_{j=1}^{n} \sup_{f
  \in \F} \frac{1}{m} \sum_{i=1}^{m} \vsigma_{j,i}
  f(s_{i}) \enspace.
\end{equation*}
Denote the \emph{Empirical Rademacher Average} (ERA) $\era$ as the expectation of the $\nmcera$ w.r.t. the assignments of the Rademacher random variables $\vsigma$, where each $\vsigma_{i,j}$ is $1$ or $-1$ independently and with equal probability:
\begin{equation*}%\label{eq:radeera}
 \era \doteq \E_\vsigma \left[ \erade^{n}_{m}(\F, \sample, \vsigma) \right] \enspace.
\end{equation*}
Then, denote the \emph{Rademacher Complexity} (RC) $\rc$ as the expectation of the ERA over $\sample$,
\begin{equation*}%\label{eq:raderc}
 \rc \doteq \E_\sample \left[ \era \right] \enspace.
\end{equation*}

The following fundamental result, also known as ``Symmetrization lemma'', show a precise relationship between the RC and the \emph{expected} supremum deviation
\citep{ShalevSBD14,mitzenmacher2017probability}.

\begin{lemma}
\begin{align*}
&\E_{\sample} \left[ \supdevpos \right] \leq 2 \rc \enspace , \\
&\E_{\sample} \left[ \supdevneg \right] \leq 2 \rc \enspace .
\end{align*}
\end{lemma}
Therefore, upper bounding the RC yields upper bounds on the expected supremum deviations; consequently, one can obtain a probabilistic upper bound on the supremum deviations on the sample $\sample$ with the application of concentration inequalities, important tools of probability theory. Most importantly, the RC can be estimated directly on the available data using the $\nmcera$. 
We now define important quantities that will appear in most of our bounds.
First, we denote the \emph{wimpy variance} $\wvar$ of $\F$ as 
\begin{align*}
\wvar \doteq \sup_{f \in \F} \brpars{ \E \sqpars{ f^2 } } \enspace .
\end{align*} 
Then, we denote the \emph{empirical wimpy variance} $\ewvar$ of $\F$ computed on $\sample$ as
\begin{align*}
\ewvar \doteq \frac{1}{m} \sup_{f \in \F} \left\lbrace \sum_{i=1}^m f(s_i)^{2} \right\rbrace \enspace .
\end{align*} 
We also define another quantity of interest $\maxabsf$, defined as the supremum mean absolute value of $\F$, computed over $\sample$, that is 
\begin{align*}
\maxabsf \doteq \frac{1}{m} \sup_{f\in \F} \left\lbrace \sum_{i=1}^m  \left| f(s_i) \right| \right\rbrace \enspace .
\end{align*}

In the next Sections we succinctly introduce the most widely used concentration inequalities methods: in Section~\ref{sec:boundeddiff} we introduce the method of bounded differences; in Section~\ref{sec:selfbfunctions} we present the definitions and recent results on self-bounding functions. The concept of self-bounding functions, as we will discuss later, are essential to prove our novel bounds. 
We remand for a more exhaustive coverage of the topic to the book of \cite{boucheron2013concentration}.

\section{Concentration Inequalities}
\label{sec:concentrdiff}
In this Section we introduce two of the most widely employed methods to prove concentration results for functions of independent random variables.
\subsection{The Method of Bounded Differences}
\label{sec:boundeddiff}

Let $X = \left(X_1 , \dots , X_n \right)$ be a vector of variables $X_i$, each taking values in a measurable set $\X$ and let $g : \X^n \rightarrow \mathbb{R}$ be a measurable function. We now introduce the \emph{bounded difference} property, that is often easy to prove in many settings.

\begin{definition}[Bounded difference property]
A function $g$ has the \emph{bounded difference property} if, for each $i$, $1\le i\le m$, there is a
  nonnegative constant $c_i$ such that:
  \begin{equation}\label{eq:bounded}
    \sup_{\substack{X_1,\dotsc,X_m \\ X_i'\in\X}}|g(X_1,\dotsc,X_m)-g(X_1,\dotsc,X_{i-1},X'_i,X_{i+1},\dotsc,X_m)|\le
    c_i\enspace.
  \end{equation}
\end{definition}

A central result is given by the following Theorem, that shows that $g(X)$ is well concentrated around its mean $\E[g(X)]$ (taken w.r.t. $X$), and that the the rate of convergence depends on the constants $c_{i}$ of the bounded difference property.

\begin{theorem}[\cite{mcdiarmid1989method}]\label{thm:mcdiarmid}
  Let $g : \X^{m} \rightarrow \R$
  be a function with the bounded difference property with constants $c_{i}$, for $1 \leq i \leq m$. Let $X_1,\dotsc,X_m$ be $m$ \emph{independent} random variables taking
  value in $\X^m$, and let $Z = g(X)$. Then it holds
  \[
    \Pr\left( Z \geq \E[Z] + t\right)\le \exp \left( - \frac{2t^2}{\sum_{i=1}^m c_i^2} \right)   \enspace .
  \]
  Also, it holds
  \[
  \Pr\left( Z \leq \E[Z] - t\right)\le \exp \left( - \frac{2t^2}{\sum_{i=1}^m c_i^2} \right) \enspace .
  \]
\end{theorem}

\subsection{Self-Bounding Functions}
\label{sec:selfbfunctions}
Self-bounding functions are an important class of \qt{well-behaved} functions that enjoys sharp concentration inequalities of their empirical estimates w.r.t. their expected values. We report their definitions and remand to~\cite{boucheron2013concentration} a more in-depth exposition of the subject. 

Let $X = \left(X_1 , \dots , X_n \right)$ be a vector of variables $X_i$, each taking values in a measurable set $\X$ and let $g : \X^n \rightarrow \mathbb{R}$ be a non-negative measurable function.
Then denote $g_i$ a function from $\X^{n-1} \rightarrow \mathbb{R}$. In the following definition, we introduce $(\alpha,\beta)$-self-bounding functions; we note that they may also be denoted by \emph{strongly} $(\alpha,\beta)$-self-bounding functions.

\begin{definition}[$(\alpha , \beta)$-self-bounding function]
A function $g$ is a $(\alpha,\beta)$\emph{-self-bounding function} if, for all $X \in \X^n$,
\begin{align*}
0 \leq g(X) - g_i(X^{(i)}) \leq 1 \enspace,
\end{align*} 
and
\begin{align*}
\sum_{i=1}^n \left( g(X) - g_i(X^{(i)}) \right) \leq \alpha g(X) + \beta \enspace ,
\end{align*} 
where $X^{(i)} = \left(X_1 , \dots , X_{i-1} , X_{i+1} , \dots , X_n \right) \in \X^{n-1}$ is obtained by dropping the $i$-th component of $X$.
\end{definition}

An often convenient choice of $g_i$ to prove that $g$ is self-bounding is
\begin{align*}
g_i(X^{(i)}) \doteq \inf_{X_i^\prime \in \X} g\left(X_1 , \dots , X_{i-1} , X_i^\prime , X_{i+1} , \dots , X_n \right) \enspace .
\end{align*} 

We now introduce \emph{weakly $(\alpha , \beta)$-self-bounding function}.

\begin{definition}[Weakly $(\alpha , \beta)$-self-bounding function]
A function $g$ is \emph{weakly} $(\alpha,\beta)$\emph{-self-bounding} if, for all $X \in \X^n$,
\begin{align*}
\sum_{i=1}^n \left( g(X) - g_i(X^{(i)}) \right)^2 \leq \alpha g(X) + \beta \enspace .
\end{align*}
\end{definition}

Note that a $(\alpha , \beta)$-self-bounding function is also a weakly $(\alpha , \beta)$-self-bounding function.

The next Theorem shows that if $g$ is self-bounding, then it is sharply concentrated w.r.t. \emph{its expectation} $\E \left[ g(X) \right]$ (taken w.r.t. $X$).

\begin{theorem}[\cite{boucheron2009concentration}]
\label{thm:boucheronsbf}
Let $X = \left(X_1 , \dots , X_n \right)$ be a vector of independent random variables, each taking values in a measurable set $\X$ and let $g : \X^n \rightarrow \mathbb{R}$ be a non-negative measurable function such that $Z=g(X)$ has finite mean $\E \left[ Z \right] < + \infty$.
Let $\alpha , \beta \geq 0$, and define $\nu=(3\alpha-1)/6$.
Denote $(\nu)_{+} = \max \left\lbrace \nu , 0 \right\rbrace$ and $(\nu)_{-} = \max \left\lbrace -\nu , 0 \right\rbrace$.

If $g$ is $(\alpha,\beta)$-self-bounding, then for all $t > 0$,
\begin{equation*}
\Pr \left( Z \geq \E \left[ Z \right] + t \right) \leq \exp \left( - \frac{t^2}{2\left( \alpha \E\left[Z \right] + \beta + (\nu)_{+} t \right)} \right) .
\end{equation*}

If $g$ is weakly $(\alpha,\beta)$-self-bounding and for all $i \leq n$, all $x \in \X$, $g_i(X^{(i)}) \leq g(x)$, then for all $t > 0$,
\begin{equation*}
\Pr \left( Z \geq \E \left[ Z \right] + t \right) \leq \exp \left( - \frac{t^2}{2 \left( \alpha \E\left[Z \right] + \beta + \alpha t/2 \right) } \right) .
\end{equation*}

If $g$ is weakly $(\alpha,\beta)$-self-bounding and $0 \leq g(X) - g_i(X^{(i)}) \leq 1$ for each $i \leq n$ and $x \in \X^n$, then for $0 < t \leq \E\left[ Z \right]$,
\begin{equation*}
\Pr \left( Z \leq \E \left[ Z \right] - t \right) \leq \exp \left( - \frac{t^2}{ 2 \left( \alpha \E\left[Z \right] + \beta + (\nu)_{-} t \right) } \right) .
\end{equation*}

Moreover, if $g$ is weakly $(\alpha,0)$-self-bounding with $0 \leq g(X) - g_i(X^{(i)}) \leq 1$ for all $i \leq n$ and $X \in \X^n$, then
 \begin{equation*}
\Pr \left( Z \leq \E \left[ Z \right] - t \right) \leq \exp \left( - \frac{t^2}{2\max\left\lbrace \alpha , 1 \right\rbrace \E \left[ Z \right] }  \right) .
\end{equation*}
\end{theorem}

A stronger result for $(1,0)$-self-bounding functions can be stated. 
\begin{theorem}[\cite{boucheron2000sharp}]
\label{thm:sbfbounds}
Let $X = \left(X_1 , \dots , X_n \right)$ be a vector of independent random variables, each taking values in a measurable set $\X$ and let $g : \X^n \rightarrow \mathbb{R}$ be a non-negative and bounded measurable function. Let $h(x) = (1+x) \ln (1+x) - x$.

If $g(X)$ is a $(1,0)$-self-bounding function, then, it holds, for $0 < t \leq \E[Z]$, 
\begin{equation*}
\Pr \left( \E \left[ Z \right] \geq Z + t \right) \leq \exp \left( - \E \left[ Z \right] h \pars{ - \frac{t}{ \E \left[ Z \right] }}  \right) ,
\end{equation*}
and, for $t > 0$,
\begin{equation*}
\Pr \left( Z \geq \E \left[ Z \right] + t \right) \leq \exp \left( - \E \left[ Z \right] h \pars{ \frac{t}{ \E \left[ Z \right] }}  \right) .
\end{equation*}
\end{theorem}

\section{Standard probabilistic bounds}
\label{sec:stdbounds}
In this Section we report standard bounds to the ERA and the SDs, that are proved using the bounded difference methods, and a standard bound for the RC based on the self-bounding property of the ERA.

\subsection{Standard Probabilistic Bound to the ERA}
\label{sec:erastdbounds}
The following result provides a probabilistic upper bound to the ERA from its estimate given by the $\nmcera$ is obtained through the application of the bounded differences method. 
\begin{theorem}
\label{thm:mceraeramcdiarmid}
\begin{equation*}
\Pr \left( \era \geq \mera + \varepsilon \right) \leq  \exp\left(\frac{-nm\varepsilon^2}{2z^2} \right).
\end{equation*}
\end{theorem}
\begin{proof}
It is simple to prove that $\mera$ has the bounded difference property with constants $c_{i} = 2z(nm)^{-1}$, for all $1 \leq i \leq nm$. Therefore, the bound follows from Theorem~\ref{thm:mcdiarmid}.
\qed
\end{proof}

\subsection{Standard probabilistic bounds to the RC}
\label{sec:standardboundRC}
A known property of the ERA is that it is a self-bounding function (see, for instance, Example 3.12 of \cite{boucheron2013concentration} and \cite{OnetoGAR13}). This implies concentration bounds, proved by \cite{boucheron2000sharp}, that are often sharper than the ones obtained through the bounded difference property.

\begin{theorem}
\label{thm:rcboundselfbounding}
Let, for $x \geq -1$, $h(x) \doteq (1+x)\log(1+x)-x$.
For all $0 < \varepsilon \leq \rc$, it holds
\begin{equation}
\label{eq:rcselbounding}
\Pr \left( \rc \geq \era + \varepsilon \right) \leq \exp\left(- \frac{m \rc}{c} h \left( - \frac{\varepsilon}{\rc} \right) \right) \leq \exp\left(\frac{-m\varepsilon^2}{2 c \rc} \right).
\end{equation}
Also, with probability $\geq 1-\delta$, it holds
%\begin{equation}
%\label{eq:rcselboundingexpl}
%\rc \leq \era + \frac{1}{2m} \left (\sqrt{\frange \left( 4m
%            \era + \frange \ln \frac{1}{\delta} \right) \ln
%      \frac{1}{\delta}} + \frange \ln \frac{1}{\delta} \right) \enspace .
%\end{equation}
\begin{equation}
\label{eq:rcselboundingexpl}
\rc \leq \era + \frac{c \ln \pars{ \frac{1}{\delta}}}{m} + \sqrt{ \pars{\frac{c \ln \pars{ \frac{1}{\delta}}}{m}}^2 + \frac{2c\ln\pars{\frac{1}{\delta}} \era }{ m } } \enspace .
\end{equation}
\end{theorem}
\begin{proof}
Equation \eqref{eq:rcselbounding} is a consequence of the self-bounding property of the ERA, and therefore follows from Theorem \ref{thm:sbfbounds} (Theorem 2.1 of~\cite{boucheron2000sharp}, see also Theorem 6.12 of \cite{boucheron2013concentration}).
Equation~\eqref{eq:rcselboundingexpl} is analogous to Theorem 3.11 of~\cite{OnetoGAR13}.
\qed
\end{proof}
From \eqref{eq:rcselboundingexpl} it is clear that, as the ERA $\era$ gets smaller, the rate of convergence for estimating the RC $\rc$ is between $\BOi{m^{-1/2}}$ and $\BOi{m^{-1}}$, an essential improvement in most cases~\citep{boucheron2013concentration}. 
This intuitively suggests why tight bounds to the ERA are useful and required to reach faster rates of convergence, something not achievable with the ``slow rate'' bound of Theorem~\ref{thm:mceraeramcdiarmid} (at least, not achievable without impractical large $n$ Monte Carlo trials), as we will show with simulations in Section~\ref{sec:experiments}.

\subsection{Standard Probabilistic Bounds to the SDs}

The following result gives standard bounds to the Supremum Deviations using their bounded difference property.
\label{sec:boundsstdsd}
\begin{theorem}
  Let $Z \doteq \sup_{f \in \F} \left\lbrace  \frac{1}{m}\sum_{j=1}^{m} f(s_{j}) - \E[f] \right\rbrace$.
  Then, it holds
    \begin{equation}
    \Pr \left( Z \geq \E\left[ Z \right] + \varepsilon \right) \leq \exp \left( - \frac{2 m \varepsilon^{2}}{c^{2}} \right) \enspace.
  \end{equation}
The same holds for $Z \doteq \sup_{f \in \F} \left\lbrace \E[f] - \frac{1}{m}\sum_{j=1}^{m} f(s_{j}) \right\rbrace$.
\end{theorem}
\begin{proof}
It is simple to show that $Z$ has the bounded difference property with constants $c_{i} = c/m$, for all $1 \leq i \leq m$. Thus, the bounds follows from Theorem~\ref{thm:mcdiarmid}.
\qed
\end{proof}

In Section \ref{sec:stdvarsdbounds} we will present a well known result that, when additional information on the variance of the functions of $\F$ is available, achieve much stronger bounds to the SDs, matching the rate of convergence of the ERA discussed in Section \ref{sec:standardboundRC}. 

\section{New probabilistic bounds to the ERA}
\label{sec:newboundserasb}
In this Section we show that a careful application of recent results for self-bounding functions allows to prove novel bounds to the ERA from the $\nmcera$, whose convergence rates depend on usually easy-to-compute functions of the elements of $\F$ on the sample $\sample$.
In Section~\ref{sec:nmceraselfboundingproperties} we show that   the $\nmcera$ is, in fact,  $(\alpha,\beta)$-self-bounding and weakly $(\alpha^{\prime},\beta^{\prime})$-self-bounding for appropriate values of $\alpha$, $\beta$, $\alpha^{\prime}$, and $\beta^{\prime}$. In Section~\ref{sec:selfberabounds} we show that this implies novel probabilistic bounds to the ERA.
First, define $\emaxf$, the ``empirical'' version of $z$, as 
\[
\emaxf = \sup_{s \in \sample , f \in \F} \abs{ f \pars{s} } \leq z.
\]

\subsection{Self-bounding properties of the $\nmcera$}
\label{sec:nmceraselfboundingproperties}
In this Section we prove self-bounding properties of the $\nmcera$. We demand the proofs to the Appendix. 

The first result states the $(\alpha,\beta)$-self-bounding property of the $\nmcera$.
\begin{restatable}{theorem}{selfboundingmcrade}
%\begin{theorem}
\label{thm:selfboundingmcrade}
Let a $n \times m$ matrix $\vsigma \in \{-1 , 1 \}^{n \times m}$, and
define the function $g(\vsigma)$ as
\begin{equation*}
g(\vsigma) \doteq nm \mera \enspace .
\end{equation*}
If $\emaxf \leq 1/2$, then $g(\vsigma)$ is a $(1,nm\maxabsf)$-self-bounding function.
%\end{theorem}
\end{restatable}
The second result regards the weakly $(\alpha,\beta)$-self-bounding property of the $\nmcera$.
\begin{restatable}{theorem}{mcerawsb}
%\begin{theorem}
\label{thm:mcerawsb}
Let a $n \times m$ matrix $\vsigma \in \{-1 , 1 \}^{n \times m}$, and
define the function $g(\vsigma)$ as
\begin{equation*}
g(\vsigma) \doteq nm \mera \enspace .
\end{equation*}
Then $g(\vsigma)$ is a weakly $(2\emaxf,2nm\maxsquaref)$-self-bounding function.
%\end{theorem}
\end{restatable}
%\begin{proof}
%Denote $g_{j,i}(\vsigma)$ as in the proof of Theorem~\ref{thm:selfboundingmcrade}.
%To prove that $g(\vsigma)$ is a weakly $(\alpha,\beta)$-self-bounding, we have to prove that, for all $\vsigma$, it holds
%\begin{equation*}
%\sum_{j=1}^{n}\sum_{i=1}^{m} \left( g(\vsigma) - g_{j,i}(\vsigma) \right)^2 \leq \alpha g(\vsigma) + \beta \enspace .
%\end{equation*}
%From the proof of Theorem~\ref{thm:selfboundingmcrade}, we have already proved that 
%\begin{align*}
%g_{j,i}(\vsigma) \geq g(\vsigma) - \vsigma_{j,i} f^\star_j(s_i)  - \left| f^\star_j(s_i) \right| \geq g(\vsigma) - 2\emaxf \enspace .
%\end{align*}
%Therefore, we observe that
%\begin{align*}
%& \sum_{j=1}^n \sum_{i=1}^m \left( g(\vsigma) - g_{j,i}(\vsigma) \right)^{2} \\
%& \leq \sum_{j=1}^n \sum_{i=1}^m \left( \vsigma_{j,i} f^\star_j(s_i) + \left| f^\star_j(s_i) \right| \right)^{2} \\
%& = \sum_{j=1}^n \sum_{i=1}^m \left( f^\star_j(s_i)^{2} + \left| f^\star_j(s_i) \right|^{2} + 2 \vsigma_{j,i} f^\star_j(s_i) \left| f^\star_j(s_i) \right| \right) \\
%& = \sum_{j=1}^n \sum_{i=1}^m \left( 2 f^\star_j(s_i)^{2} + 2 \vsigma_{j,i} f^\star_j(s_i) \left| f^\star_j(s_i) \right| \right) \\
%& \leq 2\emaxf \sum_{j=1}^n \sum_{i=1}^m \vsigma_{j,i} f^\star_j(s_i) + 2 \sum_{j=1}^n \sum_{i=1}^m f^\star_j(s_i)^{2} \\
%& = 2\emaxf g(\vsigma) + 2 \sum_{j=1}^n \sum_{i=1}^m f^\star_j(s_i)^{2} \\
%& \leq 2\emaxf g(\vsigma) + 2 n \sup_{f \in \F} \left\lbrace \sum_{i=1}^m f(s_i)^{2} \right\rbrace \\
%& = 2\emaxf g(\vsigma) + 2 n m \maxsquaref \enspace ,
%\end{align*}
%obtaining the statement.
%%\qed
%\end{proof}

\subsection{New probabilistic bounds on the ERA}
\label{sec:selfberabounds}
In this Section, we show that the self-bounding properties of the $\nmcera$ we proved yield sharp exponential concentration bounds that relate the $\nmcera$ to its expectation, the ERA, with significantly improved convergence rates w.r.t. the standard bound of Theorem~\ref{thm:mceraeramcdiarmid}. All the proofs can be found in the Appendix.

The first result is based on the self-bounding property of the $\nmcera$ we proved in Theorem~\ref{thm:selfboundingmcrade}.

\begin{restatable}{theorem}{newerabounds}
%\begin{theorem}
\label{thm:newerabounds}
Let $\vsigma \in \{-1 , 1 \}^{n \times m} $ be an $n \times m$ matrix of Rademacher random variables, such that $\vsigma_{j,i} \in \{-1 , 1 \}$ independently and with equal probability. 
Then, for all $0 < \varepsilon \leq \era$,
 \begin{equation}
\Pr \left( \era \geq \mera + \varepsilon \right) \leq \exp \left( - \frac{n m \varepsilon^2}{4 \emaxf \left( \era + \maxabsf \right) }  \right) \enspace .\label{eq:mceraerasbexpub}
\end{equation}
%\end{theorem}
\end{restatable}

The second result is based on the weakly self-bounding  property of the $\nmcera$ we proved in Theorem~\ref{thm:mcerawsb}.

\begin{restatable}{theorem}{neweraboundsw}
%\begin{theorem}
\label{thm:neweraboundsw}
Let $\vsigma \in \{-1 , 1 \}^{n \times m} $ be an $n \times m$ matrix of Rademacher random variables, such that $\vsigma_{j,i} \in \{-1 , 1 \}$ independently and with equal probability. 
Then, for all $0 < \varepsilon \leq \era$,
 \begin{equation}
\Pr \left( \era \geq \mera + \varepsilon \right) \leq \exp \left( - \frac{n m \varepsilon^2}{4 \left( \emaxf \era + \maxsquaref \right) }  \right) \enspace .\label{eq:mceraerasbub}
\end{equation}
%\end{theorem}
\end{restatable}

We may observe that the denominators of the exponents of \eqref{eq:mceraerasbexpub} and \eqref{eq:mceraerasbub} are not known a priori, but depend on the ERA $\era$, the quantity we actually want to bound. We remark that plugging an upper bound to $\era$ is sufficient for the validity of the results. To this aim, we may simply observe that 
\begin{align*}
\era = \E_\vsigma \left[ \mera \right] \leq \E_{\vsigma} \left[ \maxabsf \right] = \maxabsf \enspace ,
\end{align*}
obtaining that the r.h.s. of \eqref{eq:mceraerasbexpub} and \eqref{eq:mceraerasbub} are upper bounded by, respectively, 
\begin{equation*}
\exp \left( - \frac{n m \varepsilon^2}{8 \emaxf \maxabsf  }  \right) \enspace , \text{ and } \enspace \exp \left( - \frac{n m \varepsilon^2}{4 \left( \emaxf \maxabsf + \maxsquaref \right) }  \right) \enspace .
\end{equation*}
We now present alternative bounds that only depend on empirical quantities, that are often sharper than plugging the above upper bound to the ERA.

\begin{restatable}{theorem}{boundsempquant}
%\begin{theorem}
\label{thm:boundsempquant}
With probability $\geq 1-\delta$ it holds
\begin{align}
\label{eq:explicitboundabs}
\era &\leq \mera + \frac{2 \emaxf \ln\left( \frac{1}{\delta}\right) }{n m} + \sqrt{ \left(  \frac{2 \emaxf \ln\left( \frac{1}{\delta}\right) }{n m} \right)^{2} +  \frac{4 \emaxf \left( \mera + \maxabsf \right) \ln\left( \frac{1}{\delta}\right) }{n m} } \enspace .%\\
%&\leq \inf_{\alpha \in (0,1)} \left\lbrace \frac{1}{1-\alpha} \left[ \mera + \sqrt{\frac{4 z \maxabsf \ln\left(\frac{1}{\delta}\right)}{nm}} + \frac{z \ln\left(\frac{1}{\delta}\right)}{\alpha n m } \right] \right\rbrace \enspace .
\end{align}
Also, with probability $\geq 1-\delta$, it holds
\begin{align}
\label{eq:explicitboundwvar}
\era & \leq \mera + \frac{2 \emaxf \ln\left( \frac{1}{\delta}\right) }{n m} + \sqrt{ \left(  \frac{2 \emaxf \ln\left( \frac{1}{\delta}\right) }{n m} \right)^{2} +  \frac{4 \left( \emaxf \mera + \maxsquaref \right) \ln\left( \frac{1}{\delta}\right) }{n m} } \enspace . % \\
%& \leq \inf_{\alpha \in (0,1)} \left\lbrace \frac{1}{1-\alpha} \left[ \mera + \sqrt{\frac{4 \maxsquaref \ln\left(\frac{1}{\delta}\right)}{nm}} + \frac{z \ln\left(\frac{1}{\delta}\right)}{\alpha n m } \right] \right\rbrace \enspace .
\end{align}
%\end{theorem}
\end{restatable}
%\begin{proof}
%We prove the first inequality, as proving the second is analogous. 
%From Theorem~\ref{thm:newerabounds}, we have that, with probability $\geq 1-\delta$,
%\begin{align*}
%\era \leq  \mera + \sqrt{\frac{4 \emaxf (\maxabsf + \era) \ln\left(\frac{1}{\delta}\right)}{nm}}  \enspace .
%\end{align*}
%An upper bound to $\era$ can be obtained by finding the fixed point of the function $r(x)$
%\begin{align*}
%r(x) \doteq  \mera + \sqrt{\frac{4 \emaxf (\maxabsf + x) \ln\left(\frac{1}{\delta}\right)}{nm}}  \enspace .
%\end{align*}
%In fact, it is trivial to prove the following.
%\begin{lemma}
%\label{lemma:fixedpoint}
%Let $u,v,y \geq 0$. The fixed point of 
%\[
%r(x) = u + \sqrt{v + yx}
%\]
%is at
%\[
%x = u + \frac{y}{2} + \sqrt{\frac{y^{2}}{4} + uy + v} \enspace .
%\]
%\end{lemma}
%Thus, we apply Lemma~\ref{lemma:fixedpoint} to obtain, after simple calculations, the statement.
%\end{proof}

We remark that appropriate lower bounds to the ERA $\era$ can be similarly derived from the self-bounding properties proved in Section~\ref{sec:nmceraselfboundingproperties} and the application of Theorem~\ref{thm:boucheronsbf}.

By directly comparing the bounds we derived by Theorems~\ref{thm:newerabounds}~and~\ref{thm:neweraboundsw} with the one given by Theorem~\ref{thm:mceraeramcdiarmid}, we can conclude that the former will be tighter when at least one of the following is satisfied:
\begin{align*}
\era + \maxabsf \leq \frac{\emaxf}{2} \enspace \enspace , \enspace \enspace 
2 \emaxf \era + 2 \maxsquaref \leq \emaxf^{2} \enspace .
\end{align*}
As discussed before, since $\era \leq \maxabsf$, a sufficient condition for our results to be sharper is given by
\begin{align*}
\maxabsf \leq \frac{\emaxf}{4} \enspace \enspace , \enspace  \enspace \enspace
2 \emaxf \maxabsf + 2 \maxsquaref \leq \emaxf^{2} \enspace . %\numberthis \label{eq:condtohecksharper}
\end{align*}
In particular, our novel results allow to bound the ERA $\era$ below $\mera + \varepsilon$ with an $\varepsilon$ of the order of 
\[
\BO{ \sqrt{\pars{\mera+\maxabsf}/nm} } \enspace , \text{ or } \enspace  \BO{ \sqrt{\pars{\mera+\maxsquaref}/nm} } \enspace , 
\]
matching the rate of convergence of the ERA to the RC given by Theorem~\ref{thm:rcboundselfbounding}, 
instead of the $\BOi{ \sqrt{1/nm} }$ slow rate bound.
We also remark that, when we bound $\era$, both $\maxabsf$ and $\maxsquaref$ are deterministic quantities since the sample $\sample$ is fixed; thus, they can be used to select the probabilistic result to apply, as they do not depend on the realisation of any random variable.

\subsection{New special bounds for $n=1$}
\label{sec:newboundsnone}
An interesting case in applications is when only $n=1$ vector of Rademacher random variable is used to compute the $\nmcera$. In addition of being faster to compute than $n>1$, \cite{pellegrina2020mcrapper} show that in this case one may obtain a sharper bound to the SDs with only one and direct application of the bounded difference method, considering pairs composed by Rademacher random variables and samples of $\sample$ as i.i.d. random variables form an appropriate joint distribution. We now present a variant of their result, that upper bounds the RC instead of the SDs, that is useful to us to be compared with the novel result we prove with Theorem~\ref{thm:onedrawvaraware}.
\begin{theorem}[Theorem 4.6, \cite{pellegrina2020mcrapper}]
It holds
\begin{equation*}
    \Pr \left( \rc \geq \erade^{1}_{m}\left(\F, \sample, \vsigma\right) + \varepsilon \right) \leq \exp \left( -\frac{m\varepsilon^{2}}{2z^{2}} \right) \enspace ,
  \end{equation*}
  thus, with probability $\geq 1-\delta$, it holds
  \begin{equation*}
 \rc \leq \erade^{1}_{m}\left(\F, \sample, \vsigma\right) + z\sqrt{\frac{2 \ln\left(\frac{1}{\delta} \right) }{m}} \enspace .
\end{equation*}
\end{theorem}
They also remark that applying the result to the range centralised set of functions 
\[
\F^\oplus \doteq \left\lbrace g : g(x) \doteq f(x) - a - \frac{c}{2} , f \in \F , x \in \X \right\rbrace
\] 
is often convenient as it gives the sharpest constants in the bound (as $z$ for $\F^\oplus$ is equal to $c/2$).

We now derive an analogous but significantly sharper bound, whose convergence rate depends on the wimpy variance $\wvar$ of $\F$. Our proof, postponed to the Appendix, is based on the application of a left tail of Bousquet's inequality.
\begin{restatable}{theorem}{onedrawvaraware}
%\begin{theorem}
\label{thm:onedrawvaraware}
With probability $\geq 1-\delta$, it holds
\begin{align}
 \rc & \leq \onemera + \sqrt{\frac{2(2 z \rc + \wvar) \ln\left(\frac{1}{\delta} \right) }{m}} + \frac{z \ln\left(\frac{1}{\delta} \right)}{8m} \label{eq:nonevarianceawarebound} \\
& \leq \onemera + \sqrt{\frac{9}{8} \left( \frac{2z \logdelta}{m} \right)^{2} + \frac{2(2 z \onemera + \wvar) \logdelta }{m}} + \frac{17 z \logdelta }{8m} \enspace . \label{eq:nonevarianceawareboundexpl}
\end{align}
%\end{theorem}
\end{restatable}
%\begin{proof}
%Define the set of functions $\G$ as
%\[
%\G \doteq \left\lbrace g : g( x , \sigma ) \doteq  \sigma f(x) , f \in \F , x \in \X , \sigma \in \{-1 , 1\} \right\rbrace \enspace ,
%\] 
%where $\sigma$ is a Rademacher random variable.
%Therefore, we observe that 
%\[
%\E[g] = \E[f]\E[\sigma] = 0 \enspace , \enspace
%\sup_{g \in \G} Var(g) = \wvar \enspace , \enspace \left\Vert g \right\Vert_{\infty} \leq z \enspace .
%\]
%We now need the following left tail bound of Bousquet's inequality.
%\begin{corollary}[Corollary 12.2, \cite{boucheron2013concentration}]
%\label{thm:leftbousquet}
%Consider the setup of Theorem~\ref{thm:bousquetbound}. Then, for all $t \geq 0$, it holds
%\begin{equation*}
%\Pr \left( Z \leq \E[Z] - \sqrt{2vt} - \frac{dt}{8} \right) \leq \exp(-t) \enspace .
%\end{equation*}
%\end{corollary}
%Thus, we apply Corollary~\ref{thm:leftbousquet} to $\G$ to obtain the statement.
%%\qed
%\end{proof}
%We note that this result can be naturally combined with Theorem~\ref{thm:sdbousquetbound} to obtain sharp variance-dependent bounds on the SDs. 
We may observe that %Furthermore,
 \eqref{eq:nonevarianceawareboundexpl} may be sharper than the combined application of \eqref{eq:mceraerasbub} and \eqref{eq:rcselbounding} when $n=1$, since the empirical wimpy variance $\ewvar$ appears in \eqref{eq:mceraerasbub} with a factor $4$, while the wimpy variance in \eqref{eq:nonevarianceawareboundexpl} has a factor $2$. On the other hand, one should have (or compute on the data) an upper bound to $\wvar$ to apply the result, while Theorem~\ref{thm:neweraboundsw} only requires to compute its empirical counterpart $\ewvar$.

Nevertheless, we show that the empirical wimpy variance $\ewvar$ yields a sharp upper bound to the wimpy variance $\wvar$. 
Our analysis again relies on the powerful framework of self-bounding functions.

\begin{restatable}{theorem}{wvarestimation}
%\begin{theorem}
\label{thm:wvarestimation}
For $\varepsilon \leq \wvar$, it holds 
\begin{align}
\Pr \pars{ \wvar \geq \ewvar + \varepsilon } \leq \exp \pars{ - \frac{m \wvar}{z^2} h\pars{ -\frac{\varepsilon}{ \wvar }} } \leq \exp\pars{ - \frac{ m \varepsilon^2}{2 z^2 \wvar } } \enspace . \label{eq:wvarewvar}
\end{align}
Furthermore, with probability $\geq 1 - \delta$, it holds
\begin{align}
\wvar \leq \ewvar + \frac{z^2 \ln\pars{ \frac{1}{\delta} } }{ m } + \sqrt{ \pars{\frac{z^2 \ln\pars{ \frac{1}{\delta} } }{ m }}^2 + \frac{2 z^2 \ewvar \ln \pars{ \frac{1}{\delta}} }{ m } } \enspace . \label{eq:wvarewvarexpl}
\end{align}
%\end{theorem}
\end{restatable}

This result shows that the empirical wimpy variance $\ewvar$ is an accurate empirical estimator of the wimpy variance $\wvar$; we believe such observation may have interesting applications in establishing ``global'' fast rates of convergence of the SDs, as shown by~\cite{oneto2016global}.

\section{Variance-dependent probabilistic bounds to the Supremum Deviations}
\label{sec:stdvarsdbounds}
In this section we state a central result in \slt, due to \cite{bousquet2002bennett}, the sharpest refinement of a number of improvements of the
work of~\cite{talagrand1994sharper} on bounds on the deviation of the suprema of empirical processes.
This result can be applied to derive bounds on the supremum deviations that depend on the maximum variance $\tau \doteq \sup_{f\in \F} \left\lbrace Var(f) \right\rbrace$ of the functions of $\F$. These bounds can be dramatically sharper than the ones obtainable with the bounded differences method if $\tau$ is sufficiently smaller than its maximum possible value (equal to $c^{2}/4$ from \cite{popoviciu1935equations} inequality on variances).

We first report the result of Bousquet (in the version stated by Theorem A.1 of \cite{bartlett2005local}).
\begin{theorem}[Theorem 2.3, 
\cite{bousquet2002bennett}]
\label{thm:bousquetbound}
  Let $d>0$, $X_{i}$ be independent random variables distributed according to a probability distribution $P$, and let $\G$ be a set of functions from $\X$ to $\R$. Assume that all functions $g \in \G$ satisfy $\E[g] = 0$ and $\left\lVert g \right\rVert_{\infty} \leq d$. Let $\sigma^{2} \geq \sup_{g \in \G} Var \left( g(X_{i}) \right)$. Then, for any $x\geq0$,
\begin{equation*}
\Pr \left( Z \geq \E[Z] + x \right) \leq \exp \left( -v h \left( \frac{x}{cv} \right) \right) \enspace ,
\end{equation*}
where $Z = \sup_{g \in \F} \sum_{i=1}^{n} g(X_{i}) $, $h(x) = (1+x) \log (1+x) - x$ and $v = n \sigma^{2} + 2d\E[Z]$.\end{theorem}

Variance-dependent bounds to the SDs follow from Theorem \ref{thm:bousquetbound}.
\begin{theorem}\label{thm:sdbousquetbound}
  Let $Z \doteq \sup_{f \in \F} \left\lbrace  \frac{1}{m}\sum_{j=1}^{m} f(s_{j}) - \E[f] \right\rbrace$, and define $\tau \doteq \sup_{f\in \F} \left\lbrace Var(f) \right\rbrace$ and the function $h(x) \doteq (1+x) \ln(1+x) - x$.
  Then, it holds
    \begin{equation}\label{eq:bousquetimpl}
    \Pr \left( Z \geq \E\left[ Z \right] + \varepsilon \right) \leq \exp \left( - m \left( \tau + 2 c \E \left[ Z \right] \right) h\left( \frac{\varepsilon}{\tau + 2 c \E \left[ Z \right]} \right) \right) \enspace.
  \end{equation}
  Also, with probability at
  least $1 - \delta$, it holds
  \begin{equation}\label{eq:sdvarbound}
    Z \le \E\left[ Z \right] + \sqrt{\frac{2 \ln \left( \frac{1}{\delta} \right)
    \left( \tau + 2 c \E[ Z ] \right)}{m}}
        + \frac{ c \ln \left( \frac{1}{\delta} \right)}{3m} \enspace.
  \end{equation}
 The same results are valid for $Z \doteq \sup_{f \in \F} \left\lbrace \E[f] - \frac{1}{m}\sum_{j=1}^{m} f(s_{j}) \right\rbrace$.
\end{theorem}

\section{New probabilistic bounds to the Supremum Deviations}
\label{sec:newboundssds}
\cite{bousquet2003concentration} shows that Theorem~\ref{thm:bousquetbound} can be applied to analyze the concentration of the supremum of empirical processes for sets of functions satisfying a sub-additive property, a variant of the $(1,0)$-self-bounding property with relaxed requirements; in fact, the supremum deviation is sub-additive (see Section~6 and Lemma~C.1 of \citep{bousquet2003concentration}), but is not, in general, $(1,0)$-self-bounding. Still, in this Section we show that the supremum deviation is $(1,\beta)$-self-bounding, for appropriate values of $\beta$ that depend on the maximum and minimum expectations of the elements of $\F$. Consequently, we obtain novel bounds to the supremum deviation by applying concentration results for self-bounding functions, similarly to what we did for the $\nmcera$. 

We first prove self-bounding properties for the supremum deviations. Define $\maxexpf$ and $\minexpf$ as the gaps between the boundaries of the codomains of functions in $\F$ and their expectations, such that
\begin{align*}
\maxexpf = \sup_{f\in \F} \E[f] - a \enspace , \enspace
\minexpf = b - \inf_{f\in \F} \E[f] \enspace .
\end{align*}
\begin{restatable}{theorem}{sbsupmax}
%\begin{theorem}
\label{thm:sbsupmax}
Assume $c\leq1$.
Let $g(\sample)$ be 
\begin{align*}
g(\sample) \doteq m \supdevpos = \sup_{f\in \F} \left\lbrace \sum\limits_{\substack{j=1}}^{m} f(s_{j}) - m \E\left[f\right] \right\rbrace \enspace .
\end{align*}
Then, $g(\sample)$ is a $\left(1,m\maxexpf \right)$-self-bounding function.
%\end{theorem}
\end{restatable}

\begin{restatable}{theorem}{sbsupmin}
%\begin{theorem}
\label{thm:sbsupmin}
Assume $c\leq1$.
Let $g(\sample)$ be 
\begin{align*}
g(\sample) \doteq m \supdevneg = \sup_{f\in \F} \left\lbrace m \E\left[f\right] - \sum\limits_{\substack{j=1}}^{m} f(s_{j}) \right\rbrace \enspace .
\end{align*}
Then, $g(\sample)$ is a $\left(1,m\minexpf \right)$-self-bounding function.
%\end{theorem}
\end{restatable}

We now apply the concentration inequalities given by Theorem~\ref{thm:boucheronsbf} to obtain novel bounds on the supremum deviations. The first results regards the concentration of $\supdevpos$.
\begin{restatable}{theorem}{continequpper}
%\begin{theorem}
\label{thm:continequpper}
Let $Z$ be
\begin{align*}
Z \doteq \supdevpos = \sup_{f\in \F} \left\lbrace \frac{1}{m} \sum\limits_{\substack{j=1}}^{m} f(s_{j}) - \E\left[f\right] \right\rbrace \enspace .
\end{align*}
Then, it holds
\begin{align}
\Pr \left( Z \geq \E \left[ Z \right] + \varepsilon \right) \leq \exp \left( - \frac{m \varepsilon^2}{2 c \left( \E\left[Z \right] + \maxexpf + \varepsilon/3 \right)} \right) . \label{eq:upperfirststatement}
\end{align}
Consequently, with probability $\geq 1-\delta$,
\begin{align}
Z \leq \E \left[ Z \right] + \sqrt{ \left( \frac{c \ln\left(\frac{1}{\delta}\right)}{3m} \right)^{2} + \frac{2 c \ln \left( \frac{1}{\delta} \right) \left( \E\left[ Z \right] + \maxexpf \right) }{m} } + \frac{ c \ln\left(\frac{1}{\delta}\right)}{3m} \enspace . \label{eq:upperfirststatementexpl}
\end{align}
%\end{theorem}
\end{restatable}
%\begin{proof}
%We first observe that $g(\sample)$ is a non-negative function, for all $\sample$, since $f_{0} \in \F$. 
%Then, $g(\sample) \doteq m Z$ is $(1,m\maxexpf)$-self-bounding from Theorem~\ref{thm:sbsupmax}; therefore, we apply Theorem~\ref{thm:boucheronsbf} to obtain \eqref{eq:upperfirststatement}. The second statement follows from imposing the r.h.s. of \eqref{eq:upperfirststatement} to be $\leq \delta$.
%\end{proof}

An analogous result is valid for $\supdevneg$.
\begin{restatable}{theorem}{contineqlower}
%\begin{theorem}
\label{thm:contineqlower}
Let $Z$ be
\begin{align*}
Z \doteq \supdevneg =  \sup_{f\in \F} \left\lbrace \E\left[f\right] - \frac{1}{m} \sum\limits_{\substack{j=1}}^{m} f(s_{j}) \right\rbrace \enspace .
\end{align*}
%with $c\leq1$.
Then, it holds
\begin{align}
\Pr \left( Z \geq \E \left[ Z \right] + \varepsilon \right) \leq \exp \left( - \frac{m \varepsilon^2}{2c\left( \E\left[Z \right] + \minexpf + \varepsilon/3 \right)} \right) . \label{eq:lowerfirststatement}
\end{align}
Consequently, with probability $\geq 1-\delta$,
\begin{align}
Z \leq \E \left[ Z \right] + \sqrt{ \left( \frac{ c \ln\left(\frac{1}{\delta}\right)}{3m} \right)^{2} + \frac{2 c \ln \left( \frac{1}{\delta} \right) \left( \E\left[ Z \right] + \minexpf \right) }{m} } + \frac{ c \ln\left(\frac{1}{\delta}\right)}{3m} \enspace . \label{eq:lowerfirststatementexpl}
\end{align}
%\end{theorem}
\end{restatable}
%\begin{proof}
%We follow analogous steps taken in the proof of Theorem~\ref{thm:continequpper}. First, $g(\sample) \doteq m Z$ is $(1,m\minexpf)$-self-bounding from Theorem~\ref{thm:sbsupmin}; \eqref{eq:lowerfirststatement} follows from Theorem~\ref{thm:boucheronsbf}. The second statement is again obtained from bounding the r.h.s. of \eqref{eq:lowerfirststatement} below $ \delta$.
%%\qed
%\end{proof}

We may observe that the novel bounds we proved are less versatile than the result of Bousquet, as they may give faster convergence rates (w.r.t. the bounded difference method) for only one side of the deviation at a time (i.e., either for $\supdevpos$ or $\supdevneg$) instead of both simultaneously; 
this is because, for the same $\F$, $\maxexpf$ and $\minexpf$ cannot be both small. 
However, we observe that such results may be applicable to properly selected \emph{subsets} of $\F$, in a localized fashion. 
It is not trivial to directly compare these bounds with Bousquet's, in particular \eqref{eq:bousquetimpl} as it is implicit. However, we observed that our new bounds are slightly sharper than Bousquet's for some range of the values of the quantities involved in the equations, since some of the constants are more favourable. 
In particular, we can see that, when $c=1$, the dependence of the additive error term for $\E[Z]$ of \eqref{eq:upperfirststatementexpl} (and \eqref{eq:lowerfirststatementexpl}) on $\E[Z]$ is lower than the one in \eqref{eq:sdvarbound} by a factor $\sqrt{2}$; therefore, when the squared term dominates the error term, \eqref{eq:upperfirststatementexpl} (resp. \eqref{eq:lowerfirststatementexpl}) is smaller than \eqref{eq:sdvarbound} when $\maxexpf \leq \E[Z] + \tau$ (resp., $\minexpf \leq \E[Z] + \tau$), as we will discuss in Section~\ref{sec:experiments} with some simulations. 
Therefore, we conclude that the combination of Theorem~\ref{thm:sdbousquetbound} and our new results gives opportunities to obtain sharper bounds to the SDs of general families of functions.

These results depend on, respectively, the maximum or minimum expected values of elements of $\F$, while Bousquet's inequality requires an upper bound to their maximum variance; a problem in applications is how to handle the cases where these quantities are not known in advance: one intuitive solution is to estimate them from the data.

Regarding the maximum variance $\tau = \sup_{f\in \F} Var(f)$, we point out that the bounds $\minexpf$ and $\maxexpf$ to the expectations of $f \in \F$ may be sufficient to handle it; in fact, from \cite{bhatia2000better}, one has that, for all $f$,
\begin{equation}
\label{eq:varbdineq}
Var(f) \leq \left( b - \E[f] \right) \left( \E[f] - a \right) \enspace ,
\end{equation}
with equality when $f$ has binary codomain $\brpars{ a , b}$.
Consequently, we have that 
\[
\tau \leq \sup \left\lbrace (b-x)(x-a) : x \in \left[ \inf_{f\in\F}\E[f] ~,~ \sup_{f\in\F}\E[f] \right] \right\rbrace 
\leq \max \brpars{ \minexpf \pars{c - \minexpf} , \maxexpf \pars{c - \maxexpf} }
\leq \minexpf \maxexpf \enspace .
\]
Therefore, bounds to $\minexpf$ and $\maxexpf$ are of interest, as they suffice for the application of our results and Bousquet's inequality, and may give particularly good bounds for binary functions. Thus, in the following, we show that it is possible to sharply estimate both $\maxexpf$ and $\minexpf$ from the data, analogously to the empirical estimator for the wimpy variance we proved in Section~\ref{sec:newboundsnone}. The proofs are in the Appendix.

Define the empirical estimators $\emaxexpf$ and $\eminexpf$ of, respectively, $\maxexpf$ and $\minexpf$ as
\begin{align*}
\emaxexpf = \sup_{f\in\F} \brpars{  \frac{1}{m} \sum_{i=1}^{m} f\pars{s_i}} - a \enspace , \enspace \eminexpf = b - \sup_{f\in\F} \brpars{  \frac{1}{m} \sum_{i=1}^{m} f\pars{s_i}} \enspace .
\end{align*}
\begin{restatable}{theorem}{maxexpfestimation}
%\begin{theorem}
\label{thm:maxexpfestimation}
For $\varepsilon \leq \maxexpf$, it holds 
\begin{align}
\Pr \pars{ \maxexpf \geq \emaxexpf + \varepsilon } \leq \exp \pars{ - \frac{m \maxexpf}{c} h\pars{ -\frac{\varepsilon}{ \maxexpf }} } \leq \exp\pars{ - \frac{ m \varepsilon^2}{2 c \maxexpf } } \enspace . \label{eq:maxvarest}
\end{align}
Furthermore, with probability $\geq 1 - \delta$, it holds
\begin{align}
\maxexpf \leq \emaxexpf + \frac{c \ln\pars{ \frac{1}{\delta} } }{ m } + \sqrt{ \pars{\frac{c \ln\pars{ \frac{1}{\delta} } }{ m }}^2 + \frac{2 c \emaxexpf \ln \pars{ \frac{1}{\delta}} }{ m } } \enspace . \label{eq:maxvarestexpl}
\end{align}
%\end{theorem}
\end{restatable}
\begin{restatable}{theorem}{minexpfestimation}
%\begin{theorem}
\label{thm:minexpfestimation}
For $\varepsilon \leq \minexpf$, it holds 
\begin{align}
\Pr \pars{ \minexpf \geq \eminexpf + \varepsilon } \leq \exp \pars{ - \frac{m \minexpf}{c} h\pars{ -\frac{\varepsilon}{ \minexpf }} } \leq \exp\pars{ - \frac{ m \varepsilon^2}{2 c \minexpf } } \enspace . \label{eq:minvarest}
\end{align}
Furthermore, with probability $\geq 1 - \delta$, it holds
\begin{align}
\minexpf \leq \eminexpf + \frac{c \ln\pars{ \frac{1}{\delta} } }{ m } + \sqrt{ \pars{\frac{c \ln\pars{ \frac{1}{\delta} } }{ m }}^2 + \frac{2 c \eminexpf \ln \pars{ \frac{1}{\delta}} }{ m } } \enspace . \label{eq:minvarestexpl}
\end{align}
%\end{theorem}
\end{restatable}

\section{Simulations}
\label{sec:experiments}

In this Section we perform some simulations to compare our new bounds with standard available bounds, presented in Section \ref{sec:stdbounds}.

In Section \ref{sec:simboundsera} we compare different upper bounds to the ERA, computed from the $\nmcera$.
In Section \ref{sec:simboundsrc} we compare different approaches to bound the RC from either direct bounds from the $\nmcera$, or with intermediate bounds to the ERA.
In Section \ref{sec:experimentsbooundsSDs} we compare the variance-dependent bound on the SDs, from Bousquet's inequality, with our novel bounds, presented in Section \ref{sec:newboundssds}.

\subsection{Bounds to the ERA}
\label{sec:simboundsera}
In this Section we compare the standard bound given by the bounded difference method presented in Section \ref{sec:erastdbounds} with the novel bounds presented in this work in Section \ref{sec:newboundserasb}.

To do so, we fix $m=10^6$, $z = c = \emaxf = 1$, and we simulate some values for the $\mcera$ $\mera$ as functions of $\emaxexpf$ in the interval $[1/m , 1]$. For all such values, we compute upper bounds to the ERA $\era$ using the standard bound of Theorem~\ref{thm:mceraeramcdiarmid} and our novel bound of Equation~\ref{eq:explicitboundabs}, fixing in all cases $\delta = 0.05$.
For a given value of $\emaxexpf$, we set $\mera$ to
\begin{align}
\label{eq:simulatedvaluemcera}
\mera = \min \brpars{ \sqrt{ \frac{ \emaxexpf \ln \pars{C} }{m} } , \emaxexpf } \enspace ,
\end{align}
with $C=10^6$. 
This quantity gives a bound to the $\nmcera$ that is comparable to one given by Massart's Lemma for a family of functions of size $C$, in order to simulate a realistic rate of decay of the ERA~\citep{boucheron2013concentration}. The results for $n \in \brpars{1 , 10 , 10^2 }$ are shown in Figures~\ref{fig:sub1}-\ref{fig:sub3}. 
We also consider the worst-case $\mera = \emaxexpf$, least favourable for our bounds, in Figure~\ref{fig:sub4}.

%\begin{figure}[h]
%\caption{Example of a parametric plot ($\sin (x), \cos(x), x$)}
%\centering
%\includegraphics[width=0.7\textwidth]{./code/plot_ERA.pdf}
%\end{figure}

\begin{figure}
\centering
\begin{subfigure}{.49\textwidth}
  \centering
  \includegraphics[width=.95\linewidth]{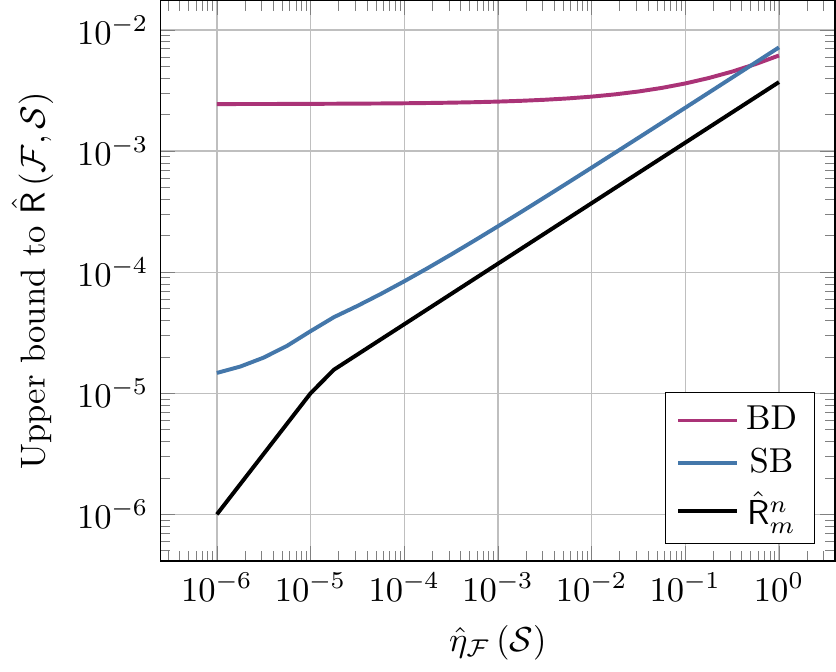}
  \caption{$n=1$.}
  \label{fig:sub1}
\end{subfigure}%
\begin{subfigure}{.49\textwidth}
  \centering
  \includegraphics[width=.95\linewidth]{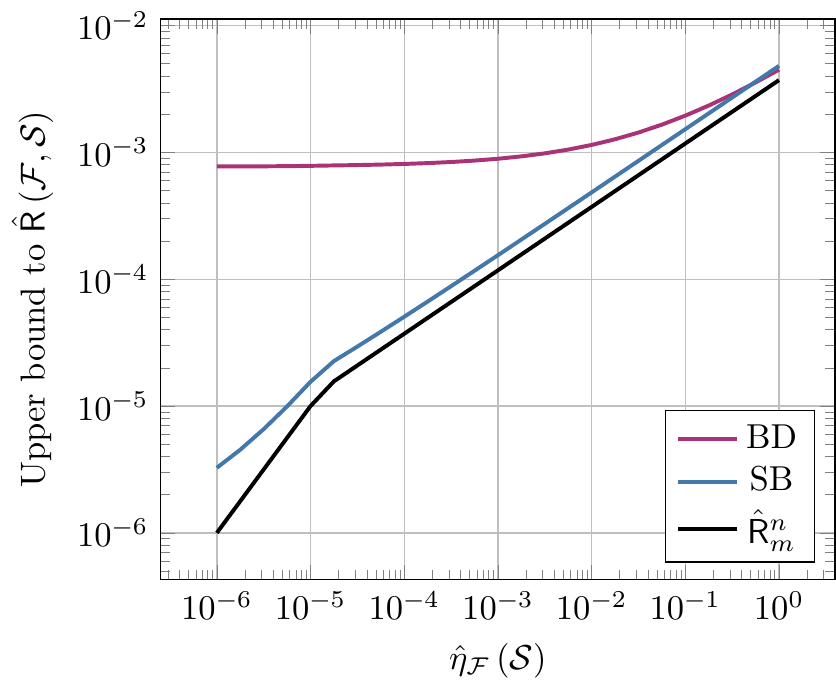}
  \caption{$n=10$.}
  \label{fig:sub2}
\end{subfigure}
\begin{subfigure}{.49\textwidth}
  \centering
  \includegraphics[width=.95\linewidth]{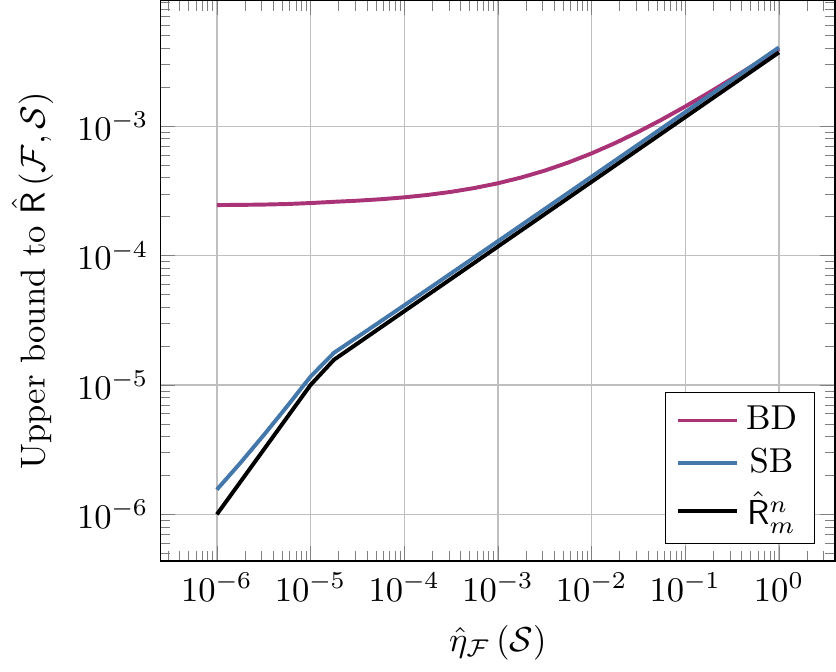}
  \caption{$n=10^2$.}
  \label{fig:sub3}
\end{subfigure}
\begin{subfigure}{.49\textwidth}
  \centering
  \includegraphics[width=.95\linewidth]{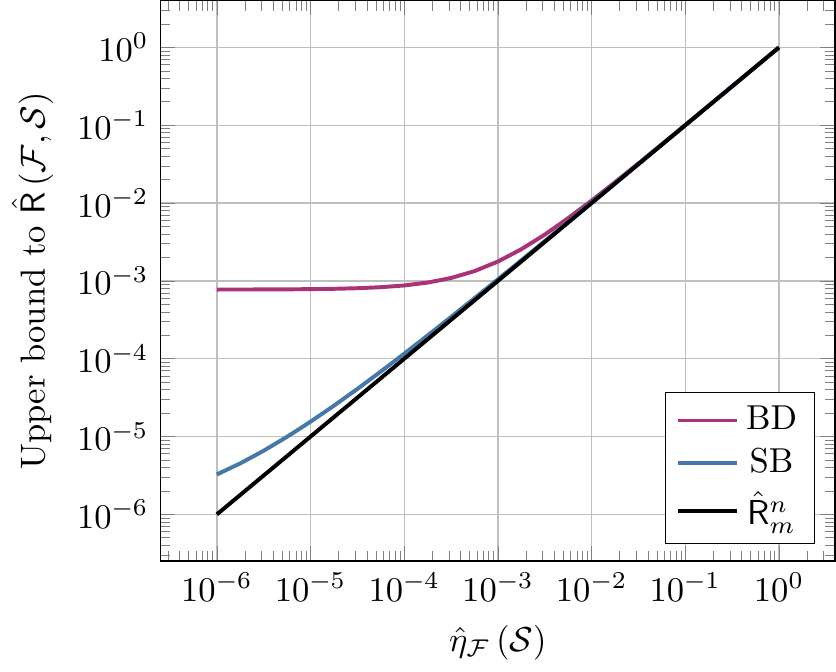}
  \caption{$n=10$, $\mera = \emaxexpf$.}
  \label{fig:sub4}
\end{subfigure}
\caption{Comparison between the upper bounds on the ERA $\era$ using the standard Bounded Difference method (in the plots with label BD, from Theorem \ref{thm:mceraeramcdiarmid}) and our novel bounds, analyzed with self-bounding functions (in the plots with label SB, Equation \ref{eq:explicitboundabs} from Theorem \ref{thm:boundsempquant}), computed from the the $\nmcera$ $\mera$ (in the plots with label $\erade^n_m$), as functions of $\emaxexpf$. We fixed $\delta = 0.05$, $m=10^6$, and fix $n$ to: (a) $n=1$, (b) $n=10$, (c) $n=10^2$, (d) $n=10$. The $\nmcera$ $\mera$ is simulated to be \eqref{eq:simulatedvaluemcera} in (a-c), and set to $\emaxexpf$ (its highest possible value) in (d).  }
\label{fig:boundsERA}
\end{figure}

Figure \ref{fig:boundsERA} clearly shows the slow rate $\BLi{\sqrt{1/(mn)}}$ given by the standard bound of Theorem \ref{thm:mceraeramcdiarmid}, as opposed to our much sharper result of Equation \ref{eq:explicitboundabs}, from Theorem \ref{thm:boundsempquant}, that instead scales with $\emaxexpf$. Analogous observations would hold for the bound of Equation \ref{eq:explicitboundwvar}, that instead depends on $\ewvar$. 

\subsection{Bounds to the RC}
\label{sec:simboundsrc}
In this Section we evaluate different upper bounds to the RC $\rc$. In particular, we compare the composition of standard bounds (Theorems \ref{thm:mceraeramcdiarmid} and \ref{thm:rcboundselfbounding}) and two alternative approaches based on our contributions: we consider the combination of the standard self-bounding bound to the RC from the ERA (Theorem \ref{thm:rcboundselfbounding}), and our novel bound of Equation \ref{eq:explicitboundabs} from the $\nmcera$ to the ERA. 
For $n=1$, we also evaluate the accuracy of the direct bound to the RC from the $\nmcera$ (Equation \ref{eq:nonevarianceawareboundexpl}), bounding the wimpy variance from the empirical wimpy variance, as showed by Theorem \ref{thm:wvarestimation} (Equation \ref{eq:wvarewvarexpl}). 
We use the same parameters of Section \ref{sec:simboundsera}, but we vary $\ewvar$ in the interval $[1/m , 1]$. 
To simulate values for $\mera$, we use again \eqref{eq:simulatedvaluemcera} replacing $\emaxexpf$ by $\ewvar$, and also consider the case $\mera = \ewvar$.
Results for different values of $n$ are shown in Figure \ref{fig:boundsRC}. 

\begin{figure}
\centering
\begin{subfigure}{.49\textwidth}
  \centering
  \includegraphics[width=.95\linewidth]{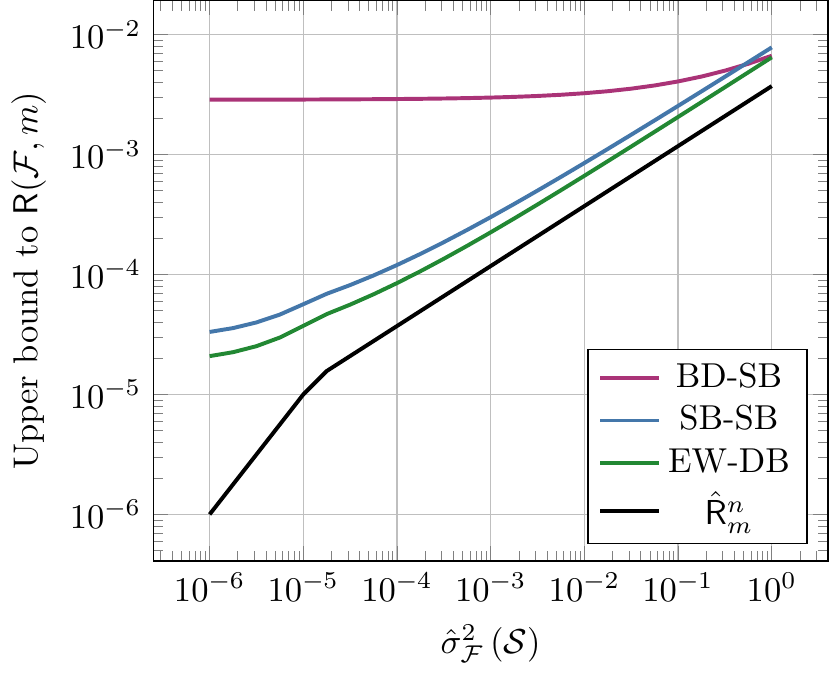}
  \caption{$n=1$.}
  \label{fig:rcsub1}
\end{subfigure}%
\begin{subfigure}{.49\textwidth}
  \centering
  \includegraphics[width=.95\linewidth]{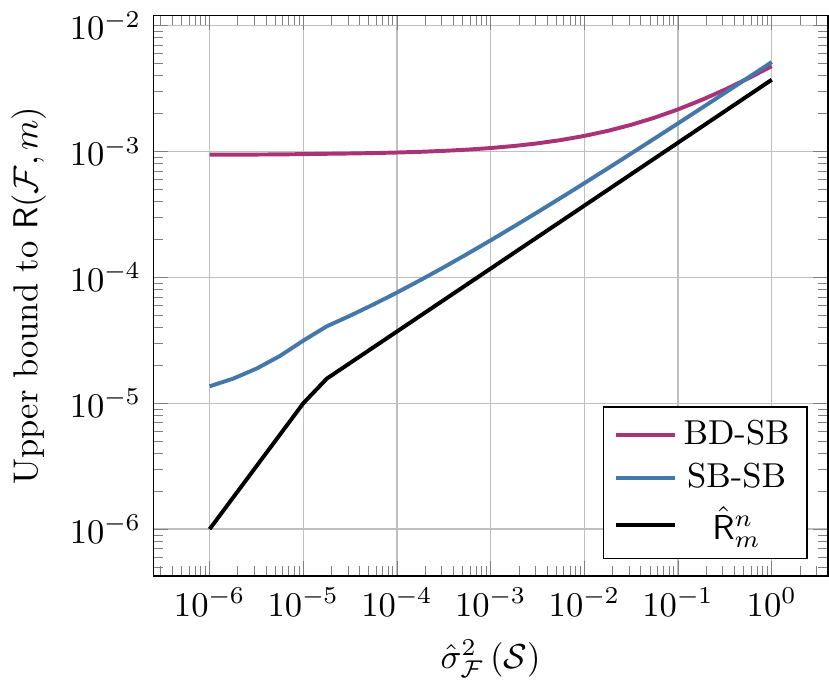}
  \caption{$n=10$.}
  \label{fig:rcsub2}
\end{subfigure}
\begin{subfigure}{.49\textwidth}
  \centering
  \includegraphics[width=.95\linewidth]{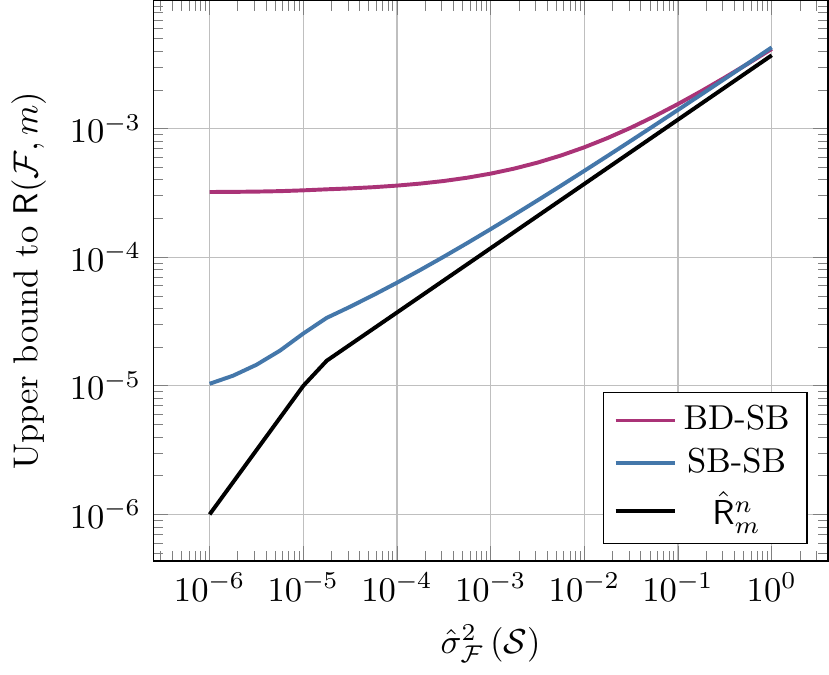}
  \caption{$n=10^2$.}
  \label{fig:rcsub3}
\end{subfigure}
\begin{subfigure}{.49\textwidth}
  \centering
  \includegraphics[width=.95\linewidth]{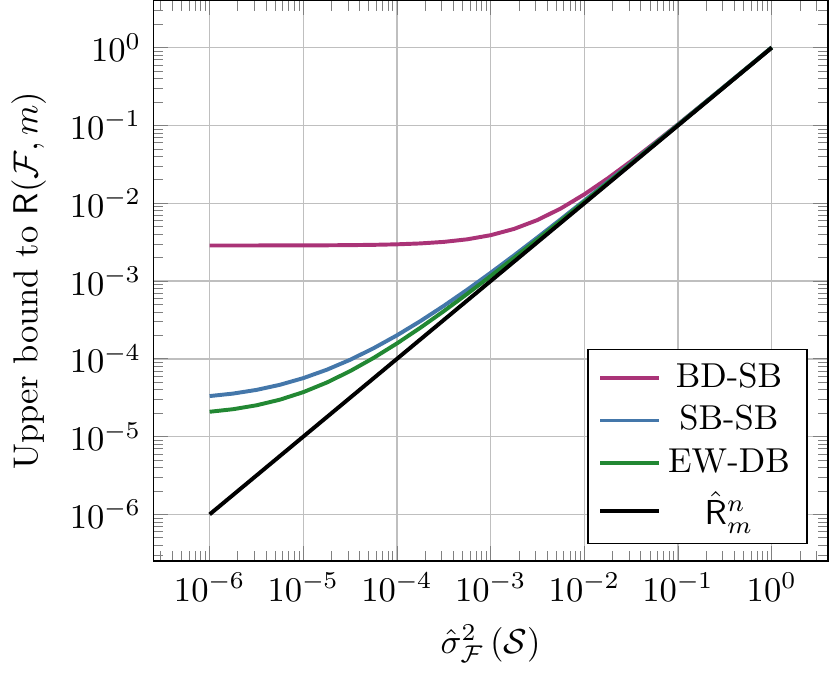}
  \caption{$n=1$, $\mera = \ewvar$.}
  \label{fig:rcsub4}
\end{subfigure}
\caption{Comparison between the upper bounds on the RC $\rc$ using standard bounds (in the plots with label BD-SB, from Theorems \ref{thm:mceraeramcdiarmid} and \ref{thm:rcboundselfbounding}), our novel bounds, analyzed with self-bounding functions (in the plots with label SB-SB, Equation \ref{eq:explicitboundabs} from Theorems \ref{thm:rcboundselfbounding} and \ref{thm:boundsempquant}), computed from the the $\nmcera$ $\mera$ (in the plots with label $\erade^n_m$), as functions of $\ewvar$. For $n=1$, we also show the direct bound of Equation \ref{eq:nonevarianceawareboundexpl}, where the wimpy variance is estimated using Equation \ref{eq:wvarewvarexpl} (in the plots with label EW-DB). We fixed $\delta = 0.05$, $m=10^6$, and fix $n$ to: (a) $n=1$, (b) $n=10$, (c) $n=10^2$, (d) $n=1$. The $\nmcera$ $\mera$ is simulated to be \eqref{eq:simulatedvaluemcera} in (a-c), and set to $\ewvar$ in (d). }
\label{fig:boundsRC}
\end{figure}

We conclude from Figure \ref{fig:boundsRC} that the slow rate given by estimating the ERA though the standard bound (BD-SD) propagates to the bound to the RC $\rc$, following the same $\BLi{\sqrt{1/(mn)}}$ trend observed in Section \ref{sec:simboundsera}, with analogous results for different values for $m$. Therefore, our novel bounds are essential to achieve faster rates of convergence for estimating the RC from the $\nmcera$. We further observe that, in the case $n=1$ (Figures \ref{fig:rcsub1} and \ref{fig:rcsub4}), our novel direct bound (EW-DB, Equation \ref{eq:nonevarianceawareboundexpl}) achieve even sharper guaranteed accuracy w.r.t. the ``full self-bounding'' approach (SB-SB), due to some of the constants being more favourable, and thanks to the sharp empirical estimator of the wimpy variance (Equation \ref{eq:wvarewvarexpl}), as we discussed in more details in Section \ref{sec:newboundsnone}. 

\subsection{Bounds to the SDs}
\label{sec:experimentsbooundsSDs}
Next, we consider our newly introduced bounds to the SDs (SB, from Theorems \ref{thm:continequpper} and \ref{thm:contineqlower}) to assess their behaviour w.r.t. the variance-dependent bound (VD), obtained through Bousquet's inequality (Theorem \ref{thm:sdbousquetbound}). 
To do so, we consider the setting of binary functions $f : \X \rightarrow \{ 0 , 1 \}$, of interest, for example, for evaluating the performance of classification models with $0$-$1$ loss. 
This setting simplifies the evaluation of $\tau = \sup_f \{ Var(f) \}$ to $\tau = \sup_f \{ \E[f] (1-\E[f]) \}$, required by Bousquet's inequality. 
We focus on bounding $\supdevpos$, using Bousquet's inequality \eqref{eq:sdvarbound} and our novel result \eqref{eq:upperfirststatementexpl}. We compute both values varying $\E[\supdevpos]$ and $\maxexpf$ over $[1/m,1/2]^2$ (so that, in such case, $\tau = \maxexpf(1-\maxexpf)$), and compute the ratio between the upper bounds to $\E[\supdevpos]$. We report in Figure contour plots (with levels $\in \{ 0.95 , 0.98 , 1 , 1.02 , 1.05 , 1.1 , 1.15 \}$) for such measurements, fixing $\delta=0.05$ and $m \in \{10^3 , 10^6 \}$. 

%\begin{figure}[h]
%\caption{Example of a parametric plot ($\sin (x), \cos(x), x$)}
%\centering
%\includegraphics[width=0.5\textwidth]{./code/plot_SD_m103.pdf}
%\end{figure}

\begin{figure}
\centering
\begin{subfigure}{.49\textwidth}
  \centering
  \includegraphics[width=.95\linewidth]{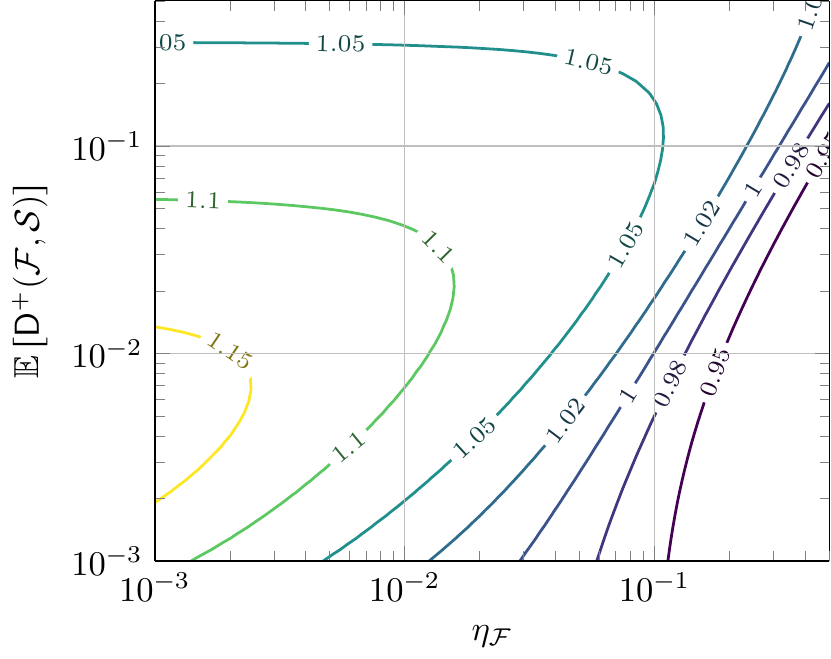}
  \caption{$m=10^3$.}
  \label{fig:sdsub1}
\end{subfigure}%
\begin{subfigure}{.49\textwidth}
  \centering
  \includegraphics[width=.95\linewidth]{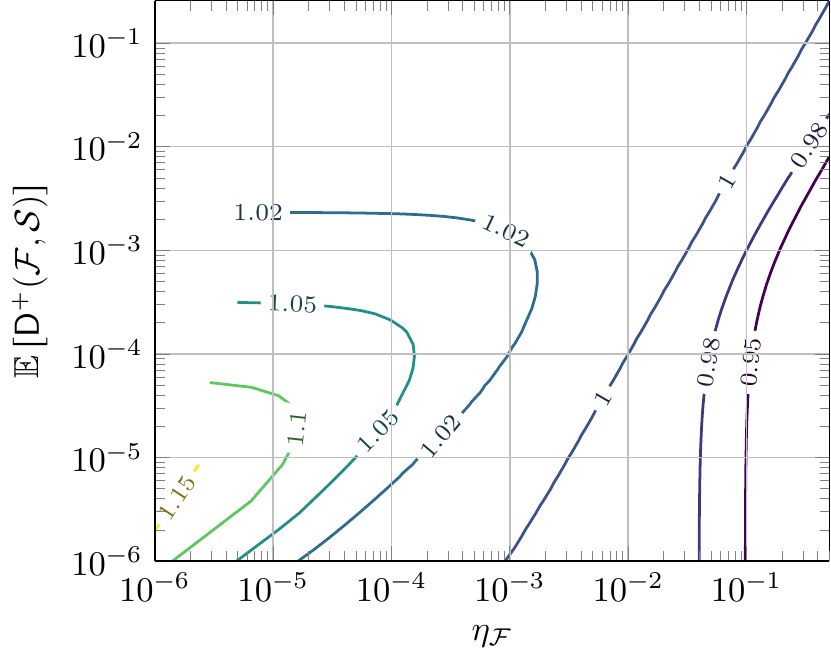}
  \caption{$m=10^6$.}
  \label{fig:sdsub2}
\end{subfigure}
%\begin{subfigure}{.49\textwidth}
%  \centering
%  \includegraphics[width=.95\linewidth]{./code/plot_RC_n100.pdf}
%  \caption{$n=10^2$.}
%  \label{fig:sdsub3}
%\end{subfigure}
%\begin{subfigure}{.49\textwidth}
%  \centering
%  \includegraphics[width=.95\linewidth]{./code/plot_RC_mcera_wc.pdf}
%  \caption{$n=1$, $\mera = \ewvar$.}
%  \label{fig:sdsub4}
%\end{subfigure}
\caption{Comparison between the upper bounds on the SD $\supdevpos$ using the variance-dependent bound of Bousquet (VD, Equation~\ref{eq:sdvarbound}) and our novel bound, analyzed with self-bounding functions (SB, Equation~\ref{eq:explicitboundabs}), for values of $\E[\supdevpos]$ and $\maxexpf$ over $[1/m,1/2]^2$. The plot shows levels of constant values of the ratio between VD and SB (VD is larger than SB for ratio $>1$). The number of samples is $m=10^3$ in (a), and $10^6$ in (b). }
\label{fig:boundsSD}
\end{figure}

From Figure \ref{fig:boundsSD} we can conclude that, as intuitively guessed in Section \ref{sec:newboundssds}, there is a region of values of $\E[\supdevpos]$ and $\maxexpf$ in which our novel bound is tighter, that is the region close to where $\maxexpf \leq \E[\supdevpos] + \tau = \E[\supdevpos] + \maxexpf(1-\maxexpf)$. 
We may see that such difference is more pronounced with relatively average size samples (Figure~\ref{fig:sdsub1}), but still present for larger samples (Figure~\ref{fig:sdsub2}).

\section{Conclusions}
In this work we studied the self-bounding properties of the $\nmcera$, and show that they allow to derive novel sharper concentration bounds w.r.t. its expectation. Obtaining tight error rates on the $\nmcera$ is of central importance to obtain tight probabilistic upper bounds on the Rademacher Averages and, therefore, uniform deviation bounds to the maximum deviation between empirical means and their expectations of sets of functions.

While in this work we focused on deriving concentration results valid with high probability in finite samples, another interesting direction is to combine the self-bounding properties we proved with asymptotical concentration results, such as the Central Limit Theorem for martingales~\citep{hall2014martingale}. 
In fact, \citet{de2019rademacher} (in their Theorem 6) have shown how to apply this result to bound the Supremum Deviation from the $\nmcera$; as they discuss, in many applications asymptotic bounds may be preferred as they may be sharper than their finite-sample counterparts, in particular when the size of the analysed data is sufficiently large and the convergence to the normal distribution is reasonably accurate.
The self-bounding properties we proved in this work imply tighter bounds on the variance of the random processes modelled by such martingales (see Chapter 6.11 of~\citet{boucheron2013concentration}); therefore, an interesting question is whether our results could enable a sharper application of the Central Limit Theorem for martingales in such setting.

Then, we remark that extending Theorem~\ref{eq:nonevarianceawarebound} to directly bound the SDs and for $n>1$ should be possible, provided that a careful analysis of the maximum variance of a properly modified set of functions (analogous to the set $\G$ defined in the proof of Theorem~\ref{thm:onedrawvaraware}) is handled. \citet{de2019rademacher} follow this idea (in their Theorems~2 and~4) to directly bound the RC or SDs from the $\nmcera$; the resulting bounds are derived by controlling the covariances of the random variables involved in their martingales. We believe that combining such derivations with our application of Bousquet's inequality is an interesting direction to explore, as, according to our simulations (Section \ref{sec:simboundsrc}), such approach seems quite promising.

Finally, we conclude by observing that there is a gap between the guaranteed concentration of $(1,0)$-self-bounding functions and general (weakly) $(\alpha,\beta)$-self-bounding functions (i.e., see the bounds of Theorems \ref{thm:boucheronsbf} and \ref{thm:sbfbounds}), as the latter do not enjoy the same Bennet-type bounds of the former; the same holds between the two sides of Bousquet's inequality (see Section 12.5 of \citep{boucheron2013concentration}, and Theorem \ref{thm:bousquetbound} and Corollary \ref{thm:leftbousquet}). Filling these gaps is an interesting and important research question. 
We point out that, in case stronger results on the concentration of self-bounding functions may be obtained, they would immediately be applicable to obtain even stronger convergence bounds for $\nmcera$, thanks to its self-bounding properties we have shown in this work.

Another fundamental and extremely interesting future reseach direction is to consider the concentration of unbounded functions~\citep{kontorovich2014concentration,mendelson2014learning,cortes2019relative,grunwald2020fast}, of great interest in many applications.

\section{Acknowledgments}
We would like to thank Fabio Vandin for fruitful discussions and precious comments that improved this manuscript.

%\bibliographystyle{splncs04}
%

\newpage 
\appendix
% !TEX root = mcrade-bounds.tex

\section{Missing proofs}
\subsection{Proof of Theorem \ref{thm:selfboundingmcrade}}
\selfboundingmcrade*
\begin{proof}
Denote the function $g_{j,i}(\vsigma)$, for $j \in [1,n]$ and $i \in [1,m]$, as
\begin{equation*}
g_{j,i}(\vsigma) \doteq \inf_{\vsigma_{j,i}^{\prime} \in \{ -1,1 \}} \left\lbrace \sum\limits_{\substack{v=1 \\ v\neq j}}^{n} \left[ \sup_{f
  \in \F}  \sum_{h = 1}^{m} \vsigma_{v,h}
  f(s_{h}) \right] + \sup_{f
  \in \F} \left\lbrace   \sum\limits_{\substack{h=1 \\ h\neq i}}^{m} \left( \vsigma_{j,h}
  f(s_{h}) \right) + \vsigma_{j,i}^{\prime} f(s_i) \right\rbrace \right\rbrace .
\end{equation*}
This function correspond to $g(\vsigma)$ where the element $\vsigma_{j,i}$ of coordinates $(i,j)$ of $\vsigma$ is replaced by $\vsigma^{\prime}_{j,i} \in \{-1,1 \}$; in addition, we take the infimum over $\vsigma^{\prime}_{j,i} \in \{ -1 , 1 \}$.
We remark that, even if $\vsigma$ is the argument of $g_{j,i}$ to simplify notation, $\vsigma_{j,i}$ never appears in the definition of $g_{j,i}(\vsigma)$, as required in the definition of self-bounding functions. 
To show that $g(\vsigma)$ is $(\alpha,\beta)$-self-bounding, according to the definition, we have to show that, for all $\vsigma \in \{-1,1\}^{n \times m}$, the inequalities
\begin{align*}
0 \leq g(\vsigma) - g_{j,i}(\vsigma) \leq 1 \enspace ,
\end{align*}
and
\begin{align*}
\sum_{j=1}^{n}\sum_{i=1}^{m}( g(\vsigma) - g_{j,i}(\vsigma) ) \leq \alpha g(\vsigma) + \beta \numberthis \label{eq:eqsumtoprove}
\end{align*}
all hold for some non-negative $\alpha$ and $\beta$.
First, $g(\vsigma) \geq g_{j,i}(\vsigma)$ follows from writing $g_{j,i}(\vsigma)$ as 
\begin{align*}
g_{j,i}(\vsigma) = \min \Bigg[ & \sum\limits_{\substack{v=1 \\ v\neq j}}^{n} \left[ \sup_{f
  \in \F}  \sum_{h = 1}^{m} \vsigma_{v,h}
  f(s_{h}) \right] + \sup_{f
  \in \F} \left\lbrace   \sum\limits_{\substack{h=1 \\ h\neq i}}^{m} \left( \vsigma_{j,h}
  f(s_{h}) \right) - f(s_i) \right\rbrace , \\
  &\sum\limits_{\substack{v=1 \\ v\neq j}}^{n} \left[ \sup_{f
  \in \F}  \sum_{h = 1}^{m} \vsigma_{v,h}
  f(s_{h}) \right] + \sup_{f
  \in \F} \left\lbrace   \sum\limits_{\substack{h=1 \\ h\neq i}}^{m} \left( \vsigma_{j,h}
  f(s_{h}) \right) + f(s_i) \right\rbrace
   \Bigg] \enspace ,
\end{align*} 
and from the observation that one argument of the $\min$ is equal to $g(\vsigma)$, therefore the minimum is either equal to $g(\vsigma)$ or $< g(\vsigma)$.
We now prove that, if $z \leq 1/2$, $g(\vsigma) \leq g_{j,i}(\vsigma) + 1$, for all $\vsigma$ and for all $j$ and $i$.
\begin{align*}
g_{j,i}(\vsigma) 
&= \inf_{\vsigma_{j,i}^{\prime} \in \{ -1,1 \}} \left\lbrace \sum\limits_{\substack{v=1 \\ v\neq j}}^{n} \left[ \sup_{f
  \in \F} \sum_{h = 1}^{m} \vsigma_{v,h}
  f(s_{h}) \right] +  \sup_{f
  \in \F} \left\lbrace \sum\limits_{\substack{h=1 \\ h\neq i}}^{m} \left( \vsigma_{j,h}
  f(s_{h}) \right) + \vsigma_{j,i}^{\prime} f(s_i) \right\rbrace \right\rbrace \\   
  &=  \sum\limits_{\substack{v=1 \\ v\neq j}}^{n} \left[ \sup_{f
  \in \F} \sum_{h=1}^{m} \vsigma_{v,h}
  f(s_{h}) \right] +  \inf_{\vsigma_{j,i}^{\prime} \in \{ -1,1 \}} \left\lbrace   \sup_{f
  \in \F} \left\lbrace \sum\limits_{\substack{h=1 \\ h\neq i}}^{m} \left( \vsigma_{j,h}
  f(s_{h}) \right) + \vsigma_{j,i}^{\prime} f(s_i) \right\rbrace \right\rbrace \\
  &\geq  \sum\limits_{\substack{v=1 \\ v\neq j}}^{n} \left[ \sup_{f
  \in \F} \sum_{h=1}^{m} \vsigma_{v,h}
  f(s_{h}) \right] +   \sup_{f
  \in \F} \left\lbrace \inf_{\vsigma_{j,i}^{\prime} \in \{ -1,1 \}} \left\lbrace  \sum\limits_{\substack{h=1 \\ h\neq i}}^{m} \left( \vsigma_{j,h}
  f(s_{h}) \right) + \vsigma_{j,i}^{\prime} f(s_i) \right\rbrace \right\rbrace \\
  &=  \sum\limits_{\substack{v=1 \\ v\neq j}}^{n} \left[ \sup_{f
  \in \F}  \sum_{h = 1}^{m} \vsigma_{v,h}
  f(s_{h}) \right] +   \sup_{f
  \in \F} \left\lbrace \sum\limits_{\substack{h=1 \\ h\neq i}}^{m} \left( \vsigma_{j,h}
  f(s_{h}) \right) + \inf_{\vsigma_{j,i}^{\prime} \in \{ -1,1 \}} \left\lbrace \vsigma_{j,i}^{\prime} f(s_i) \right\rbrace \right\rbrace .
\end{align*}
Let $f^\star_j$ be one of the functions of $\F$ attaining the supremum of $\sup_{f
  \in \F} \sum_{h = 1}^{m} \vsigma_{j,h}
  f(s_{h})$. Then we continue 
\begin{align*}
g_{j,i}(\vsigma) &\geq  \sum\limits_{\substack{v=1 \\ v\neq j}}^{n} \left[ \sup_{f
  \in \F}  \sum_{h = 1}^{m} \vsigma_{v,h}
  f(s_{h}) \right] +   \sup_{f
  \in \F} \left\lbrace  \sum\limits_{\substack{h=1 \\ h\neq i}}^{m} \left( \vsigma_{j,h}
  f(s_{h}) \right) + \inf_{\vsigma_{j,i}^{\prime} \in \{ -1,1 \}} \left\lbrace \vsigma_{j,i}^{\prime} f(s_i) \right\rbrace \right\rbrace \\
  &\geq  \sum\limits_{\substack{v=1 \\ v\neq j}}^{n} \left[ \sup_{f
  \in \F}  \sum_{h = 1}^{m} \vsigma_{v,h}
  f(s_{h}) \right] + \sum\limits_{\substack{h=1 \\ h\neq i}}^{m} \left( \vsigma_{j,h}
  f^\star_j(s_{h}) \right) +  \inf_{\vsigma_{j,i}^{\prime} \in \{ -1,1 \}} \left\lbrace \vsigma_{j,i}^{\prime} f^\star_j(s_i)  \right\rbrace  \\
  &=  \sum\limits_{\substack{v=1 \\ v\neq j}}^{n} \left[ \sup_{f
  \in \F}  \sum_{h = 1}^{m} \vsigma_{v,h}
  f(s_{h}) \right] + \sum\limits_{\substack{h=1 \\ h\neq i}}^{m} \left( \vsigma_{j,h}
  f^\star_j(s_{h}) \right) + \vsigma_{j,i} f^\star_j(s_i) - \vsigma_{j,i} f^\star_j(s_i) +  \inf_{\vsigma_{j,i}^{\prime} \in \{ -1,1 \}} \left\lbrace \vsigma_{j,i}^{\prime} f^\star_j(s_i)  \right\rbrace \\
  &=  g(\vsigma) - \vsigma_{j,i} f^\star_j(s_i) +  \inf_{\vsigma_{j,i}^{\prime} \in \{ -1,1 \}} \left\lbrace \vsigma_{j,i}^{\prime} f^\star_j(s_i)  \right\rbrace \enspace . \numberthis \label{eq:tosimplify}
\end{align*}
We first observe that
\begin{align*}
 \inf_{\vsigma_{j,i}^{\prime} \in \{ -1,1 \}} \left\lbrace \vsigma_{j,i}^{\prime} f^\star_j(s_i)  \right\rbrace = 
\begin{cases}
0 & \text{ , if } f^\star_j(s_i) = 0 \enspace , \\
 -f^\star_j(s_i) & \text{ , if } f^\star_j(s_i) > 0 \enspace , \\
f^\star_j(s_i) & \text{ , if } f^\star_j(s_i) < 0 \enspace ,
\end{cases} 
\end{align*}
obtaining 
\begin{align*}
 \inf_{\vsigma_{j,i}^{\prime} \in \{ -1,1 \}} \left\lbrace \vsigma_{j,i}^{\prime} f^\star_j(s_i)  \right\rbrace = - \left| f^\star_j(s_i) \right| \enspace .
\end{align*}
Therefore, we continue from \eqref{eq:tosimplify} as follows:
\begin{align*}
g_{j,i}(\vsigma) \geq g(\vsigma) - \vsigma_{j,i} f^\star_j(s_i)  - \left| f^\star_j(s_i) \right| \geq g(\vsigma) - 2\emaxf \geq g(\vsigma) - 1 \enspace .
\end{align*}
We now prove \eqref{eq:eqsumtoprove} for $\alpha = 1$ and $\beta = nm\maxabsf$.
\begin{align*}
& \sum_{j=1}^n \sum_{i=1}^m \left( g(\vsigma) - g_{j,i}(\vsigma) \right) \\
& \leq\sum_{j=1}^n \sum_{i=1}^m \left( g(\vsigma) - g(\vsigma) + \vsigma_{j,i} f^\star_j(s_i) + \left| f^\star_j(s_i) \right| \right) \\
& =\sum_{j=1}^n \sum_{i=1}^m \left( \vsigma_{j,i} f^\star_j(s_i) + \left| f^\star_j(s_i) \right| \right) \\
& = g(\vsigma) + \sum_{j=1}^n \sum_{i=1}^m \left| f^\star_j(s_i) \right| \\
& \leq g(\vsigma) + n \sup_{f\in \F} \left\lbrace \sum_{i=1}^m  \left| f(s_i) \right| \right\rbrace \\
& = g(\vsigma) + n m \maxabsf  \enspace , \numberthis \label{eq:lasteqsbmcera}
\end{align*} 
concluding the proof.
\qed
\end{proof}

\subsection{Proof of Theorem \ref{thm:mcerawsb}}
\mcerawsb*
\begin{proof}
Denote $g_{j,i}(\vsigma)$ as in the proof of Theorem~\ref{thm:selfboundingmcrade}.
To prove that $g(\vsigma)$ is a weakly $(\alpha,\beta)$-self-bounding, we have to prove that, for all $\vsigma$, it holds
\begin{equation*}
\sum_{j=1}^{n}\sum_{i=1}^{m} \left( g(\vsigma) - g_{j,i}(\vsigma) \right)^2 \leq \alpha g(\vsigma) + \beta \enspace .
\end{equation*}
From the proof of Theorem~\ref{thm:selfboundingmcrade}, we have already proved that 
\begin{align*}
g_{j,i}(\vsigma) \geq g(\vsigma) - \vsigma_{j,i} f^\star_j(s_i)  - \left| f^\star_j(s_i) \right| \geq g(\vsigma) - 2\emaxf \enspace .
\end{align*}
Therefore, we observe that
\begin{align*}
& \sum_{j=1}^n \sum_{i=1}^m \left( g(\vsigma) - g_{j,i}(\vsigma) \right)^{2} \\
& \leq \sum_{j=1}^n \sum_{i=1}^m \left( \vsigma_{j,i} f^\star_j(s_i) + \left| f^\star_j(s_i) \right| \right)^{2} \\
& = \sum_{j=1}^n \sum_{i=1}^m \left( f^\star_j(s_i)^{2} + \left| f^\star_j(s_i) \right|^{2} + 2 \vsigma_{j,i} f^\star_j(s_i) \left| f^\star_j(s_i) \right| \right) \\
& = \sum_{j=1}^n \sum_{i=1}^m \left( 2 f^\star_j(s_i)^{2} + 2 \vsigma_{j,i} f^\star_j(s_i) \left| f^\star_j(s_i) \right| \right) \\
& \leq 2\emaxf \sum_{j=1}^n \sum_{i=1}^m \vsigma_{j,i} f^\star_j(s_i) + 2 \sum_{j=1}^n \sum_{i=1}^m f^\star_j(s_i)^{2} \\
& = 2\emaxf g(\vsigma) + 2 \sum_{j=1}^n \sum_{i=1}^m f^\star_j(s_i)^{2} \\
& \leq 2\emaxf g(\vsigma) + 2 n \sup_{f \in \F} \left\lbrace \sum_{i=1}^m f(s_i)^{2} \right\rbrace \\
& = 2\emaxf g(\vsigma) + 2 n m \maxsquaref \enspace ,
\end{align*}
obtaining the statement.
\qed
\end{proof}

\subsection{Proof of Theorem \ref{thm:newerabounds}}
\newerabounds*
\begin{proof}
Define the set of functions 
\begin{align*}
\F^{\prime} \doteq \left\lbrace f^{\prime}(x) \doteq f(x) / (2\emaxf) : \forall x \in \X ,  f \in \F \right\rbrace \enspace ,
\end{align*}
composed by all functions $f \in \F$ divided by $2\emaxf$; clearly, $|f^{\prime}(s)| \leq 1/2$, $\forall s \in \sample$. We now show that $n m \merad$ (consequently, also $n m \mera$) is a non-negative function:
\begin{align*}
 n m \merad \doteq \sum_{j=1}^{n} \sup_{f^{\prime}
  \in \F^{\prime}} \sum_{s_{i} \in \sample} \vsigma_{j,i}
  f^{\prime}(s_{i}) \geq \sum_{j=1}^{n} \sum_{s_{i} \in \sample} \vsigma_{j,i}
  f_{0}(s_{i}) = 0 \enspace .
\end{align*}
From Theorem~\ref{thm:selfboundingmcrade}, we have that $n m \merad$ is a $(1,mn\maxabsfd)$-self-bounding function.
This implies that it is also a weakly $(1,mn\maxabsfd)$-self-bounding function.
Then, note that $\E_\vsigma \left[ n m \merad \right] = n m \erad$. We combine these facts with Theorem~\ref{thm:boucheronsbf}, obtaining, for $g(\vsigma) = nm\merad$,
 \begin{equation*}
\Pr \left( n m \erad \geq n m \merad + t \right) \leq \exp \left( - \frac{t^2}{2 \left( n m \erad + n m \maxabsfd \right) }  \right) \enspace .
\end{equation*}
We observe that $\erad = \era/(2\emaxf)$, $\merad = \mera/(2\emaxf)$, and that $\maxabsfd = \maxabsf/(2\emaxf)$. We make these substitutions, obtaining 
 \begin{equation*}
\Pr \left( \frac{n m}{2\emaxf} \era \geq \frac{n m}{2\emaxf} \mera + t \right) \leq \exp \left( - \frac{\emaxf t^2}{\left( n m \era + n m \maxabsf \right) }  \right) \enspace .
\end{equation*}
We further substitute $t$ by $nm \varepsilon / (2\emaxf)$, obtaining the statement.
\qed
\end{proof}

\subsection{Proof of Theorem \ref{thm:neweraboundsw}}
\neweraboundsw*
\begin{proof}
Let $\F^{\prime}$ be the same set of functions defined in the proof of Theorem~\ref{thm:newerabounds}.
If we denote $g(\vsigma) \doteq nm\merad$, then, from Theorem~\ref{thm:mcerawsb}, $g(\vsigma)$ is a weakly $(1,2mn \maxsquarefd )$-self-bounding function.
As before, $\E_\vsigma \left[ n m \merad \right] = n m \erad$. 
We apply Theorem~\ref{thm:boucheronsbf} on $g(\vsigma) = nm\merad$, obtaining 
 \begin{equation*}
\Pr \left( n m \erad \geq n m \merad + t \right) \leq \exp \left( - \frac{t^2}{2 \left( n m \erad + 2 n m \maxsquarefd \right) }  \right) \enspace .
\end{equation*}
We observe that $\erad = \era/(2\emaxf)$, $\merad = \mera/(2\emaxf)$, and that $\maxsquarefd = \maxsquaref/(4\emaxf^{2})$. This implies that
 \begin{equation*}
\Pr \left( \frac{n m}{2\emaxf} \era \geq \frac{n m}{2\emaxf} \mera + t \right) \leq \exp \left( - \frac{t^2}{\left( \frac{n m}{\emaxf} \era + \frac{n m}{\emaxf^{2}} \maxsquaref \right) }  \right) \enspace .
\end{equation*}
Replacing $t$ by $\varepsilon nm/(2\emaxf)$ concludes the proof.
\qed
\end{proof}

\subsection{Proof of Theorem \ref{thm:boundsempquant}}
\boundsempquant*
\begin{proof}
We prove the first inequality, as proving the second is analogous. 
From Theorem~\ref{thm:newerabounds}, we have that, with probability $\geq 1-\delta$,
\begin{align*}
\era \leq  \mera + \sqrt{\frac{4 \emaxf (\maxabsf + \era) \ln\left(\frac{1}{\delta}\right)}{nm}}  \enspace .
\end{align*}
An upper bound to $\era$ can be obtained by finding the fixed point of the function $r(x)$
\begin{align*}
r(x) \doteq  \mera + \sqrt{\frac{4 \emaxf (\maxabsf + x) \ln\left(\frac{1}{\delta}\right)}{nm}}  \enspace .
\end{align*}
In fact, it is trivial to prove the following.
\begin{lemma}
\label{lemma:fixedpoint}
Let $u,v,y \geq 0$. The fixed point of 
\[
r(x) = u + \sqrt{v + yx}
\]
is at
\[
x = u + \frac{y}{2} + \sqrt{\frac{y^{2}}{4} + uy + v} \enspace .
\]
\end{lemma}
Thus, we apply Lemma~\ref{lemma:fixedpoint} to obtain, after simple calculations, the statement.
\end{proof}

\subsection{Proof of Theorem \ref{thm:onedrawvaraware}}
\onedrawvaraware*
\begin{proof}
Define the set of functions $\G$ as
\[
\G \doteq \left\lbrace g : g( x , \sigma ) \doteq  \sigma f(x) , f \in \F , x \in \X , \sigma \in \{-1 , 1\} \right\rbrace \enspace ,
\] 
where $\sigma$ is a Rademacher random variable.
Therefore, we observe that, from independence of the random variables $\sigma$ and $f(x)$, 
\begin{align*}
&\E[g] = \E[f]\E[\sigma] = 0 \enspace , \enspace \left\Vert g \right\Vert_{\infty} = \sup \brpars{ \abs{ \sigma f(x)} : \sigma \in \brpars{-1,1} ,  x \in \X , f \in \F } \leq z \enspace , \\
&\sup_{g \in \G} Var(g) = \sup_{f \in \F} \brpars{ \E\sqpars{\sigma^2}\E\sqpars{f^2} - \pars{\E\sqpars{\sigma}\E\sqpars{f}}^2 }
= \sup_{f \in \F} \brpars{ \E\sqpars{\sigma^2}\E\sqpars{f^2} }
= \sup_{f \in \F} \brpars{ \E\sqpars{f^2} } = \wvar \enspace .
\end{align*}
We now need the following left tail bound of Bousquet's inequality.
\begin{corollary}[Corollary 12.2, \cite{boucheron2013concentration}]
\label{thm:leftbousquet}
Consider the setup of Theorem~\ref{thm:bousquetbound}. Then, for all $t \geq 0$, it holds
\begin{equation*}
\Pr \left( Z \leq \E[Z] - \sqrt{2vt} - \frac{dt}{8} \right) \leq \exp(-t) \enspace .
\end{equation*}
\end{corollary}
Thus, we apply Corollary~\ref{thm:leftbousquet} to $\G$ to obtain \eqref{eq:nonevarianceawarebound}. The bound of \eqref{eq:nonevarianceawareboundexpl} follows from Lemma~\ref{lemma:fixedpoint}.
\qed
\end{proof}

\subsection{Proof of Theorem \ref{thm:wvarestimation}}
\wvarestimation*
\begin{proof}
We first prove that 
\begin{align*}
\wvar \leq \E_\sample \sqpars{\ewvar}
\end{align*}
by observing, through Jensen's inequality, that
\begin{align*}
\wvar = \sup_{f\in\F} \brpars{ \E \sqpars{f^2} }
= \sup_{f\in\F} \brpars{ \E_{\sample} \sqpars{ \frac{1}{m} \sum_{i=1}^{m} \pars{f\pars{s_i}}^2} }
\leq \E_{\sample} \sqpars{ \sup_{f\in\F} \brpars{  \frac{1}{m} \sum_{i=1}^{m} \pars{f\pars{s_i}}^2} }
= \E_{\sample} \sqpars{ \ewvar } \enspace .
\end{align*}
We now show that $\ewvar$ is a $(1,0)$-self-bounding function. 
Let the function $g(\sample) = m \ewvar$, and, for $j \in [1,m]$, let the function $g_j(\sample)$ be
\begin{align*}
g_j(\sample) = \sup_{f\in\F} \brpars{ \sum\limits_{\substack{i=1 \\ i\neq j}}^{m} \pars{f\pars{s_i}}^2} \enspace .
\end{align*} 
First, it holds $g(\sample) \geq 0$, and $g_j(\sample) \leq g(\sample)$, for all $\sample$ and all $j$, as $\pars{f\pars{s}}^2 \geq 0, \forall s$.
We now prove that $g(\sample) - g_j(\sample) \leq z^2$. Let $f^\star$ be one of the functions of $\F$ attaining the supremum for $g(\sample)$; then,  
\begin{align*}
g_j(\sample) 
= \sup_{f\in\F} \brpars{ \sum\limits_{\substack{i=1 \\ i\neq j}}^{m} \pars{f\pars{s_i}}^2}
\geq \sum\limits_{\substack{i=1 \\ i\neq j}}^{m} \pars{f^\star\pars{s_i}}^2
= \sum_{i=1}^{m} \pars{f^\star\pars{s_i}}^2 - \pars{f^\star\pars{s_j}}^2 
= g(\sample) - \pars{f^\star\pars{s_j}}^2 
\geq g(\sample) - z^2 \enspace .
\end{align*}
Consequently, we have
\begin{align*}
\sum_{j=1}^{m} \pars{ g(\sample) - g_j(\sample) }
\leq \sum_{j=1}^{m} \pars{ \pars{f^\star\pars{s_j}}^2 }
= g(\sample) \enspace ,
\end{align*}
that concludes the proof that $\ewvar$ is a $(1,0)$-self-bounding function.
We now apply Theorem \ref{thm:sbfbounds} to obtain a probabilistic bounds to the expectation $\E_\sample[\ewvar]$ of $\ewvar$; we have 
\begin{align*}
\Pr \pars{ \E_\sample[\ewvar] \geq \ewvar + \varepsilon } \leq \exp \pars{ - \frac{\E_\sample[\ewvar]}{z^2} h\pars{ - \frac{\varepsilon}{ \E_\sample[\ewvar] }} } \enspace .
\end{align*}
The fact that $\wvar \leq \E_\sample[\ewvar]$ has two implications: first, we have that 
\begin{align*}
\Pr \pars{ \wvar \geq \ewvar + \varepsilon } \leq \Pr \pars{ \E_\sample[\ewvar] \geq \ewvar + \varepsilon } \enspace ;
\end{align*}
then, due to the monotonicity of $-x h(- \varepsilon/x)$ in $x$, we have
\begin{align*}
\exp \pars{ - \frac{\E_\sample[\ewvar]}{z^2} h\pars{ - \frac{\varepsilon}{ \E_\sample[\ewvar] }} } \leq 
\exp \pars{ - \frac{\wvar}{z^2} h\pars{ -  \frac{\varepsilon}{ \wvar }} } \enspace ,
\end{align*}
obtaining the first bound of \eqref{eq:wvarewvar}.
The second follows from the fact that $h(-x) \geq x^2/2, \forall x \in [0,1]$, as pointed out by \citet{boucheron2000sharp}.
The inequality \eqref{eq:wvarewvarexpl} follows from bounding the rightmost term of \eqref{eq:wvarewvar} below $\delta$, and by applying Lemma~\ref{lemma:fixedpoint}. 
\qed
\end{proof}

\subsection{Proof of Theorem \ref{thm:sbsupmax}}
\sbsupmax*
\begin{proof}
Let $g_{i}\left(\sample \right)$ be
\begin{align*}
g_{i}\left(\sample \right) 
&= \inf_{s_{i}^{\prime}} \left\lbrace \sup_{f\in \F} \left\lbrace \sum\limits_{\substack{j=1 \\ j\neq i}}^{m} f(s_{j}) + f(s_{i}^{\prime}) - m \E\left[f\right] \right\rbrace \right\rbrace \enspace .
\end{align*}
Notice that, as done before, $s_{i}$ is ignored in the definition of $g_{i}\left(\sample \right) $.
Let $f^{\star}$ be one of the functions in $\F$ that attains the supremum for $g(\sample)$. We then have
\begin{align*}
g_{i}\left(\sample \right) 
&= \inf_{s_{i}^{\prime}} \left\lbrace \sup_{f\in \F} \left\lbrace \sum\limits_{\substack{j=1 \\ j\neq i}}^{m} f(s_{j}) + f(s_{i}^{\prime}) - m \E\left[f\right] \right\rbrace \right\rbrace \\
&\geq \inf_{s_{i}^{\prime}} \left\lbrace \sum\limits_{\substack{j=1 \\ j\neq i}}^{m} f^{\star}(s_{j}) + f^{\star}(s_{i}^{\prime}) - m \E\left[f^{\star}\right]  \right\rbrace \\
&= \sum\limits_{\substack{j=1 \\ j\neq i}}^{m} f^{\star}(s_{j}) - m \E\left[f^{\star}\right] + \inf_{s_{i}^{\prime}} \left\lbrace  f^{\star}(s_{i}^{\prime}) \right\rbrace \\
&= \sum\limits_{\substack{j=1 \\ j\neq i}}^{m} f^{\star}(s_{j}) - m \E\left[f^{\star}\right] + a \\
&= \sum\limits_{\substack{j=1 \\ j\neq i}}^{m} f^{\star}(s_{j}) + f^{\star}(s_{i}) - f^{\star}(s_{i}) - m \E\left[f^{\star}\right] + a \\
&= g(\sample) - f^{\star}(s_{i}) + a \enspace .
\end{align*}
We then observe that $g_{i}\left(\sample \right) \geq g(\sample) - b + a = g(\sample) - c$; assuming that $c\leq1$, we have $g_{i}\left(\sample \right) \geq g(\sample) - 1$. We then continue with
\begin{align*}
& \sum_{j=1}^m \left( g(\sample) - g_{i}(\sample) \right) \\
&\leq \sum_{j=1}^m \left( f^{\star}(s_{i}) - a \right) \\
&= \sum_{j=1}^m f^{\star}(s_{i}) - am \\
&= \sum_{j=1}^m f^{\star}(s_{i}) -m\E\left[ f^{\star} \right] + m\E\left[ f^{\star} \right] - am \\
&= g(\sample) + m \E\left[ f^{\star} \right] - m a \\
&\leq g(\sample) + m \maxexpf \enspace ,
\end{align*} 
obtaining the statement.
\qed
\end{proof}

\subsection{Proof of Theorem \ref{thm:sbsupmin}}
\sbsupmin*
\begin{proof}
Define the set of functions $\F^{\prime} \doteq \left\lbrace f^{\prime}(x) \doteq -f(x) , f \in \F , x \in \X \right\rbrace$. We have that $f^{\prime} \in [-b , -a]$, that $\E[f^{\prime}] = - \E[f]$, and that $\sum\limits_{\substack{j=1}}^{m} f^{\prime}(s_{j}) = -\sum\limits_{\substack{j=1}}^{m} f(s_{j})$. Therefore,
\begin{align*}
g(\sample) = \sup_{f\in \F} \left\lbrace m \E\left[f\right] - \sum\limits_{\substack{j=1}}^{m} f(s_{j}) \right\rbrace  = \sup_{f^{\prime}\in \F^{\prime}} \left\lbrace \sum\limits_{\substack{j=1}}^{m} f^{\prime}(s_{j}) - m \E\left[f^{\prime}\right] \right\rbrace \enspace .
\end{align*}
Then, we may observe that 
\begin{align*}
\minexpf = b - \inf_{f\in \F} \E[f] = b + \sup_{f\in \F} \E[f^{\prime}] =  \sup_{f\in \F} \E[f^{\prime}] - \min_{x} f^{\prime}(x) \enspace .
\end{align*}
Thus, we apply Theorem~\ref{thm:sbsupmax} to $g(\sample)$ and $\F^{\prime}$ to show that it is $(1,m\minexpf)$-self bounding, obtaining the statement.
\qed
\end{proof}

\subsection{Proof of Theorem \ref{thm:continequpper}}
\continequpper*
\begin{proof}
We first observe that $g(\sample)$ is a non-negative function, for all $\sample$, since $f_{0} \in \F$. 
Then, $g(\sample) \doteq m Z$ is $(1,m\maxexpf)$-self-bounding from Theorem~\ref{thm:sbsupmax}; therefore, we apply Theorem~\ref{thm:boucheronsbf} to obtain \eqref{eq:upperfirststatement}. The second statement follows from imposing the r.h.s. of \eqref{eq:upperfirststatement} to be $\leq \delta$.
\qed
\end{proof}

\subsection{Proof of Theorem \ref{thm:contineqlower}}
\contineqlower*
\begin{proof}
We follow analogous steps taken in the proof of Theorem~\ref{thm:continequpper}. First, $g(\sample) \doteq m Z$ is $(1,m\minexpf)$-self-bounding from Theorem~\ref{thm:sbsupmin}; \eqref{eq:lowerfirststatement} follows from Theorem~\ref{thm:boucheronsbf}. The second statement is again obtained from bounding the r.h.s. of \eqref{eq:lowerfirststatement} below $ \delta$.
\qed
\end{proof}

\subsection{Proof of Theorem \ref{thm:maxexpfestimation}}
\maxexpfestimation*
\begin{proof}
We follow similar steps taken in the proof of Theorem \ref{thm:wvarestimation}. 
We first prove that 
\begin{align*}
\maxexpf \leq \E_\sample \sqpars{\emaxexpf}
\end{align*}
by observing, through Jensen's inequality, that
\begin{align*}
\maxexpf = \sup_{f\in\F} \brpars{ \E \sqpars{f} } - a
= \sup_{f\in\F} \brpars{ \E_{\sample} \sqpars{ \frac{1}{m} \sum_{i=1}^{m} f\pars{s_i}} } - a
\leq \E_{\sample} \sqpars{ \sup_{f\in\F} \brpars{  \frac{1}{m} \sum_{i=1}^{m} f\pars{s_i}} } - a
= \E_{\sample} \sqpars{ \emaxexpf } \enspace .
\end{align*}
We now show that $\emaxexpf$ is a self-bounding function. 
Let the function $g(\sample) = m \emaxexpf$, and, for $j \in [1,m]$, let the function $g_j(\sample)$ be
\begin{align*}
g_j(\sample) = \inf_{s_j^\prime} \brpars{ \sup_{f\in\F} \brpars{ \sum\limits_{\substack{i=1 \\ i\neq j}}^{m} f\pars{s_i} + f\pars{s_j^\prime}} } - a \enspace .
\end{align*} 
First, it holds $g(\sample) \geq 0$, as $f\pars{s} \geq a, \forall s$, and $g_j(\sample) \leq g(\sample)$ by definition of $g_j(\sample)$.
We now prove that $g(\sample) - g_j(\sample) \leq c$. Let $f^\star$ be one of the functions of $\F$ attaining the supremum for $g(\sample)$; then,  
\begin{align*}
g_j(\sample) 
= \inf_{s_j^\prime} \brpars{ \sup_{f\in\F} \brpars{ \sum\limits_{\substack{i=1 \\ i\neq j}}^{m} f\pars{s_i} + f\pars{s_j^\prime}} } - a
\geq \sum\limits_{\substack{i=1 \\ i\neq j}}^{m} f^\star\pars{s_i} + \inf_{s_j^\prime} \brpars{ f^\star\pars{s_j^\prime}} - a
= \sum_{i=1}^{m} f^\star\pars{s_i} - f^\star\pars{s_j} \\
= g(\sample) - f^\star\pars{s_j} + a
\geq g(\sample) - c \enspace .
\end{align*}
Consequently, we have
\begin{align*}
\sum_{j=1}^{m} \pars{ g(\sample) - g_j(\sample) }
\leq \sum_{j=1}^{m} \pars{ f^\star\pars{s_j} - a }
= g(\sample) \enspace ,
\end{align*}
that concludes the proof that $\emaxexpf$ is a $(1,0)$-self-bounding function.
We now apply Theorem \ref{thm:sbfbounds} to a family of functions that is scaled by $1/c$ (i.e., as we did in the proof of Theorem \ref{thm:newerabounds}) to obtain a probabilistic bounds to the expectation $\E_\sample[\emaxexpf]$ of $\emaxexpf$; we have 
\begin{align*}
\Pr \pars{ \E_\sample[\emaxexpf] \geq \emaxexpf + \varepsilon } \leq \exp \pars{ - \frac{\E_\sample[\emaxexpf]}{c} h\pars{ -  \frac{\varepsilon}{ \E_\sample[\emaxexpf] }} } \enspace .
\end{align*}
As in the proof of Theorem \ref{thm:wvarestimation}, $\maxexpf \leq \E_\sample[\emaxexpf]$ implies that 
\begin{align*}
\Pr \pars{ \maxexpf \geq \emaxexpf + \varepsilon } \leq \Pr \pars{ \E_\sample[\emaxexpf] \geq \emaxexpf + \varepsilon } \enspace ,
\end{align*}
and 
\begin{align*}
\exp \pars{ - \frac{\E_\sample[\emaxexpf]}{c} h\pars{ -  \frac{\varepsilon}{ \E_\sample[\emaxexpf] }} } \leq 
\exp \pars{ - \frac{\maxexpf}{c} h\pars{ \frac{\varepsilon}{ \maxexpf }} } \enspace .
\end{align*}
By combining these two observations, we obtain the first bound of \eqref{eq:maxvarest}.
The second follows from the fact that $h(-x) \geq x^2/2, \forall x \in [0,1]$, as pointed out by \citet{boucheron2000sharp}.
The inequality \eqref{eq:maxvarestexpl} follows from bounding the rightmost term of \eqref{eq:wvarewvar} below $\delta$, and by applying Lemma~\ref{lemma:fixedpoint}. 
\qed
\end{proof}

\subsection{Proof of Theorem \ref{thm:minexpfestimation}}
\minexpfestimation*
\begin{proof}
First, we define the set of functions $\F^\prime$ as
\begin{align*}
\F^\prime = \brpars{ f^\prime : f^\prime(x) = -f(x), f \in \F , x \in \X } \enspace ,
\end{align*}
and we define $a^\prime = \inf_{x} f^\prime(x) = -b$, $b^\prime = \sup_{x} f^\prime(x) = -a$. We have $\eta_{\F^\prime} = \sup_{f^\prime} \E[f^\prime] - a^\prime = \gamma_{\F}$, and $\emaxexpfp = \eminexpf$.
Thus, we apply Theorem \ref{thm:maxexpfestimation} and Lemma~\ref{lemma:fixedpoint} to $\F^\prime$, obtaining, after appropriate substitutions, all the statements of the Theorem for $\F$.
\qed
\end{proof}

\newpage 
\bibliographystyle{apalike}
\bibliography{bibliography}

\iffalse
\input{variance-bounds}
\input{old-stuff}
\fi

\end{document}